\documentclass{article}

\usepackage[final]{nips_2018}

\usepackage[utf8]{inputenc} % allow utf-8 input
\usepackage[T1]{fontenc}    % use 8-bit T1 fonts
\usepackage{hyperref}       % hyperlinks
\usepackage{url}            % simple URL typesetting
\usepackage{booktabs}       % professional-quality tables
\usepackage{amsfonts}       % blackboard math symbols
\usepackage{nicefrac}       % compact symbols for 1/2, etc.
\usepackage{microtype}      % microtypography

\usepackage[usenames,dvipsnames]{xcolor}
\RequirePackage{color}
\definecolor{mydarkblue}{rgb}{0,0.08,0.45}
\hypersetup{
    colorlinks=true,
    linkcolor=red,
    filecolor=blue,      
    urlcolor=mydarkblue,
    citecolor=mydarkblue,
}

\usepackage{tikz}

\usepackage{graphicx} % more modern
\usepackage{algorithm}
\usepackage{algorithmic}
\usepackage{multirow}

\usepackage{hyperref}

\usepackage{amsmath,amsthm} % Required for Math Stuff
\usepackage{mathtools}
\usepackage{color} % Required for color
\usepackage{enumerate,textpos}
\usepackage{amssymb}
\usepackage{wrapfig}

\renewcommand{\hat}{\widehat}
\newcommand{\RR}{\mathrm{R}}
\newcommand{\cD}{\mathcal{D}}
\newcommand{\cL}{\mathcal{L}}
\newcommand{\cF}{\mathcal{F}}

\newcommand{\cP}{\mathcal{P}}
\newcommand{\cE}{\mathcal{E}}

\newcommand{\cX}{\mathcal{X}}

\DeclareMathOperator*{\argmin}{arg\,min}
\newcommand{\sI}{\I}
\newcommand{\sA}{\A}
\newcommand{\fI}{\mathfrak{I}}
\newcommand{\fA}{\mathfrak{A}}
\newcommand{\fL}{\mathfrak{L}}
\newcommand{\fq}[1]{q_{#1}(1/\delta, \log m, \log n)}
\newcommand{\fqq}[2]{q_{#1}({#2}/\delta, \log m, \log n)}
\newcommand{\bE}{\mathbb{E}}
\newcommand{\R}{\mathbb{R}}
\newcommand{\bP}{\mathbb{P}}

\newcommand{\0}{\mathrm{0}}

\newcommand{\A}{\mathrm{A}}

\newcommand{\B}{\mathrm{B}}

\renewcommand{\b}{\mathrm{b}}

\newcommand{\C}{\mathrm{C}}
\renewcommand{\d}{\mathrm{d}}
\newcommand{\D}{\mathrm{D}}

\newcommand{\e}{\mathrm{e}}
\newcommand{\E}{\mathrm{E}}
\newcommand{\f}{\mathrm{f}}

\newcommand{\cH}{\mathcal{H}}

\newcommand{\I}{\mathrm{I}}

\newcommand{\K}{\mathrm{K}}
\renewcommand{\L}{\mathrm{L}}

\newcommand{\p}{\mathrm{p}}
\renewcommand{\P}{\mathrm{P}}
\newcommand{\q}{\mathrm{q}}
\newcommand{\Q}{\mathrm{Q}}
\renewcommand{\r}{\mathrm{r}}

\renewcommand{\S}{\mathrm{S}}

\renewcommand{\u}{\mathrm{u}}
\newcommand{\U}{\mathrm{U}}
\renewcommand{\v}{\mathrm{v}}
\newcommand{\V}{\mathrm{V}}
\newcommand{\w}{\mathrm{w}}
\newcommand{\W}{\mathrm{W}}
\newcommand{\x}{\mathrm{x}}
\newcommand{\X}{\mathrm{X}}

\newcommand{\y}{\mathrm{y}}
\newcommand{\Y}{\mathrm{Y}}

\newcommand{\z}{\mathrm{z}}
\newcommand{\Z}{\mathrm{Z}}

\newcommand{\ip}[2]{\left\langle{#1},{#2}\right\rangle}

\newcommand{\cA}{\mathcal{A}}

\newcommand{\gap}{\operatorname{gap}}
\newcommand{\norm}[1]{\left\| #1 \right\|}

\renewcommand{\log}[1]{\operatorname{log}\left(#1\right)}
\renewcommand{\exp}[1]{\operatorname{exp}\left(#1\right)}
\newcommand{\trace}[1]{\normalfont\textrm{Tr}\left(#1\right)}
\newcommand{\diag}[1]{\operatorname{diag}\left(#1\right)}

\newcommand{\expectation}[2]{\mathbb{E}_{#1}\left[#2\right]}
\newcommand{\expect}[1]{\mathbb{E}\left[#1\right]}

\newcommand{\maximize}[3]{
\begin{aligned}
& \underset{#1}{\textrm{maximize}}
& & #2 \\
& \textrm{subject to}
& &  #3
\end{aligned}
}

\renewcommand{\maximize}[3]{
\begin{aligned}
& \underset{#1}{\operatorname{max}}
& & #2 \\
& \textrm{subject to}
& &  #3
\end{aligned}
}

\renewcommand{\maximize}[3]{
\begin{aligned}
 {\textrm{maximize }}
&  #2 
&\hspace*{-5pt} \textrm{s.t. }
& #3
\end{aligned}
}

\newtheorem{theorem}{Theorem}[section]

\newtheorem{lemma}[theorem]{Lemma}
\newtheorem{prop}[theorem]{Proposition}
\newtheorem{corollary}[theorem]{Corollary}

\theoremstyle{definition}
\newtheorem{definition}[theorem]{Definition}
\newtheorem{assumption}[theorem]{Assumption}
\usetikzlibrary{arrows.meta}

\title{Streaming~Kernel~PCA~with~$\tilde{O}(\sqrt{n})$~Random~Features}

\renewcommand{\footnotemark}{$^\dagger$} 

\author{
  Enayat Ullah 
   \thanks{$\dagger$ Department of Computer Science, Johns Hopkins University, Baltimore, MD 21204}\\
  \texttt{enayat@jhu.edu} \\
  \And
 Poorya Mianjy  \footnotemark \\
  \texttt{mianjy@jhu.edu} \\
  \And
  Teodor V. Marinov  \footnotemark \\
  \texttt{tmarino2@jhu.edu} \\
  \And
  Raman Arora \footnotemark \\
  \texttt{arora@cs.jhu.edu} \\
}

\begin{document}
\maketitle

\vspace*{-10pt}
\begin{abstract}
We study the statistical and computational aspects of kernel principal component analysis using random Fourier features and show that under mild assumptions, $O(\sqrt{n} \log n)$ features suffice to achieve $O(1/\epsilon^2)$ sample complexity.
Furthermore, we give a memory efficient streaming algorithm based on classical Oja's algorithm that achieves this rate.  
\end{abstract} 
\vspace*{-15pt}

\section{Introduction}
% \vspace*{-5pt}
Kernel methods represent an important class of machine learning algorithms that  simultaneously enjoy strong theoretical guarantees as well as empirical performance. 
However, it is notoriously hard to scale them to large datasets due to space and runtime complexity (typically $O(n^2)$ and $O(n^3)$, respectively, for most problems)~\citep{smola1998learning}. There have been many efforts to overcome these computational challenges, including Nystr\"om method~\citep{williams2001using}, incomplete Cholesky factorization~\citep{fine2001efficient}, random Fourier features (RFF)~\citep{rahimi2007random} and randomized sketching~\citep{yang2015randomized}. In this paper, we focus on random Fourier features due to its broad applicability to a large class of kernel problems. 

In a seminal paper by~\cite{rahimi2007random}, the authors appealed to Bochner's theorem to argue that any shift-invariant kernel can be approximated as $k(\x,\y)\approx \langle \z(\x),\z(\y) \rangle$, where the random Fourier feature mapping $\z: \R^d \to \R^m$ is obtained by sampling from the inverse Fourier transform of the kernel function. This allows one to invoke fast linear techniques to solve the linear problem in $\R^m$. However, subsequent work analyzing kernel methods based on RFF for learning problems suggests that to achieve the same asymptotic rates (as obtained using the true kernel) on the excess risk, one requires $m = \Omega(n)$ random features~\citep{rahimi2009weighted}, which defeats the purpose of using random features from a computational perspective and fails to explain its empirical success.

Last year at NIPS, while Rahimi and Recht won the test-of-time award for their work on RFF~\citep{rahimi2007random}, \cite{rudi2017generalization} showed for the first time that at least for the kernel ridge regression problem, under some mild distributional assumptions and for appropriately chosen regularization parameter, one can achieve minimax optimal statistical rates using only $m=O(\sqrt{n}\log n)$ random features. It is then natural to ask if the same holds for other kernel problems.

In this paper, we focus on Kernel Principal Component Analysis (KPCA) \citep{scholkopf1998nonlinear}, which is a popular technique for unsupervised nonlinear representation learning. We argue that scalability is an even bigger issue in the unsupervised setting since big data is largely unlabeled. Furthermore, when extending the results from the supervised learning to unsupervised learning we have to deal with additional challenges stemming from the non-convexity of the KPCA problem. We pose KPCA as a stochastic optimization problem and investigate the tradeoff between statistical samples and random features needed to guarantee  $\epsilon$-suboptimality on the population objective (aka a small generalization error).

KPCA entails computing the top-$k$ principal components of the data mapped into a Reproducing Kernel Hilbert Space (RKHS) induced by a positive definite kernel~\citep{aronszajn1950theory}. In \cite{scholkopf1998nonlinear}, authors showed that given a sample of $n$ i.i.d. draws from the underlying distribution, the infinite dimensional problem (over RKHS) can be reduced to a finite dimensional problem (in $\R^n$) using the kernel trick. In particular, the solution entails computing the top-$k$ eigenvectors of the kernel matrix computed on the given sample.
Statistical consistency of this approach was established in  \cite{shawe2005eigenspectrum} and further improved in \cite{blanchard2007statistical}. However, computational aspects of KPCA are less well understood. 
Note that the eigendecomposition of the kernel matrix alone requires $O(kn^2)$ computation, which can be prohibitive for large datasets. Several recent works have attempted to accelerate KPCA using random features. In \cite{lopez2014randomized}, authors show that the kernel matrix computed using random features converges to the true kernel matrix in operator norm at a rate of $O(n\sqrt{{(\textrm{log}~n)}/{m}})$.
In \cite{ghashami2016streaming}, authors extended this guarantee to a streaming setting using the Frequent Direction algorithm~\citep{liberty2013simple} on random features. In a related line of work,~\cite{xie2015scale} propose a stochastic optimization algorithm based on doubly stochastic gradients with a  $1/n$ convergence in the sense of angle between subspaces.  However, all these results require $m=\tilde \Omega(n)$ random features to guarantee a $O(1/\sqrt n)$ generalization bound.

More recently, \cite{sriperumbudur2017statistical} studied statistical consistency of ERM with randomized Fourier features. They showed that the top-k eigenspace of the empirical covariance matrix in the random feature space converges to that of the population covariance operator in the RKHS when lifted to the space of square integrable functions, at a rate of $O(1/\sqrt{m}+1/\sqrt{n})$ \footnote{While our paper was under review,   \cite{sriperumbudur2017statistical}, which initially focused on statistical consistency of kernel PCA with random features, was replaced by \cite{sriperumbudur2018approximate}, with a new title and focus on computational and statistical tradeoffs of KPCA much like our paper.}.
This result suggests that statistical and computational efficiency cannot be achieved at the same time without making further assumptions. In this paper, we assume a spectral decay on the distribution of the data in the feature space to show that we can simultaneously guarantee spectral and computational efficiency for KPCA using random features. Our main contributions are as follows. 

\begin{table}[t]
\vspace*{-25pt}
\begin{center}
 \begin{tabular}{||c | c | c | c | c ||}
 \hline
  Algorithm & Reference & Sample complexity & Per-iteration cost & Memory \\ [0.5ex] 
 \hline\hline
 \multirow{2}{*}{ERM} &  {\small \cite{shawe2005eigenspectrum}} & $\tilde{O}(1/\epsilon^2)$ & $\tilde{O}(k/\epsilon^4)$ & $O(1/\epsilon^4)$  \\ [0.5ex]
 \cline{2-5}
   & {\small \cite{blanchard2007statistical}}$^\dagger$ & $\tilde{O}(1/\epsilon)$ & $\tilde{O}(k/\epsilon^2)$ & $O(1/\epsilon^2)$ \\ 
 \hline
    RF-DSG  & {\small \cite{xie2015scale}} & $\tilde{O}(1/\epsilon^2)$ & $\tilde{O}(k/\epsilon^2)$ & $O(k/\epsilon^2)$ \\
 \hline
\multirow{2}{*}{RF-ERM} & {\small \cite{lopez2014randomized}}& $\tilde{O}(1/\epsilon^2)$ & $\tilde{O}(k/\epsilon^4)$ & $\tilde{O}(1/\epsilon^4)$ \\ 
\cline{2-5}
  & {\small 
  {\bf{Corollary~\ref{cor:main_erm_oja}}}$^\dagger$}& $\tilde{O}(1/\epsilon^2)$ & $\tilde{O}(k/\epsilon^3)$ & $\tilde{O}(1/\epsilon^{3})$ \\ 
 \hline
 RF-Oja & {\small {\bf{Corollary~\ref{cor:main_erm_oja}$^\dagger$}}}& $\tilde{O}(1/\epsilon^2)$ & $\tilde{O}(k/\epsilon)$ & ${{\tilde{O}(k/\epsilon)}}$ \\
 \hline
\end{tabular}
\end{center}
\caption{Comparing different approaches to KPCA in terms of sample complexity, per-iteration computational cost and space complexity. $\dagger:$ Optimistic rates realized under (potentially different) higher-order distributional assumptions (See Corollary ~\ref{cor:main_good_decay}, and \cite{blanchard2007statistical}).}
    \label{tab:comparison}
\end{table}

\vspace*{-5pt}
\begin{enumerate}
    \item We study kernel PCA as stochastic optimization problem and show that under mild distributional assumptions, for a wide range of kernels, the empirical risk minimizer (ERM) in the random feature space converges in objective as $O(1/\sqrt n)$ whenever $m=\Omega(k\sqrt{n}\log n)$, with overall runtime of $O(kn^{\frac32}\log n)$.
    % \vspace*{-2pt}
    \item We propose a stochastic approximation algorithm based on classical Oja's updates on random features which enjoys the same statistical guarantees as the ERM above but with better runtime and space requirements. 
    \item We overcome a key challenge associated with kernel PCA using random features which is to ensure that the output of the algorithm corresponds to a projection operator in the (potentially infinite dimensional) RKHS. We establish that the output of the proposed algorithms converges to a projection operator. 
    \item In order to better understand the computational benefits of using random features, we also consider the KPCA problem in a streaming setting, where at each iteration, the algorithm is provided with a fresh sample drawn i.i.d. from the underlying distribution and is required to output a solution based on the samples observed so far. In such a setting, comparison with other algorithmic approaches suggests that Oja's algorithm on random Fourier features (see RF-Oja in Table~\ref{tab:comparison}) enjoys the best overall runtime as well as superior space complexity.
    \item We contribute novel analytical tools that should be useful broadly when designing algorithms for kernel methods based on random features.~We provide crucial and novel insights 
    that exploit connections between covariance operators in RKHS and the space of square integrable functions with respect to data distribution.~This connection allows us to look at the kernel \textit{approximation} using random features as an \textit{estimation} problem in the space of square integrable functions, where we appeal to recent results in local Rademacher complexity~\citep{massart2000some,bartlett2002localized,blanchard2007statistical} to yield faster rates.
    \item Finally, we provide empirical results on a real dataset to support our theoretical results. 
\end{enumerate}

The rest of the paper is organized as follows. In Section~\ref{sec:prob_form}, we give the problem setup. In Section~\ref{sec:prelim}, we provide mathematical preliminaries and introduce the key notation. The main algorithm and the results are in Section~\ref{sec:main_results} and the empirical results are discussed in Section~\ref{sec:experiments}. 

% \vspace*{-7pt}
\section{Problem setup}
\label{sec:prob_form}
% \vspace*{-7pt}
Given a random vector $\x \in \R^d$
 with underlying distribution $\rho$, 
 principal component analysis (PCA) can be formulated as the following stochastic optimization problem~\citep{arora2012stochastic,arora2013stochastic}:
\begin{equation}\label{prob:pca}
    \maximize{\P}{\bE_{\x\sim\rho}{\langle\P,\x\x^\top\rangle}}{\P \in \cP^k},
\end{equation}
where $\cP^k$ is the set of $d\times d$ rank-$k$ orthogonal projection matrices.
Essentially, PCA seeks a $k$-dimensional subspace of $\R^d$ that captures maximal variation with respect to the underlying distribution. It is well understood that the solution to the problem above is given by the projection matrix corresponding to the subspace spanned by the top-$k$ eigenvectors of the covariance matrix $\expect{\x\x^\top}$.

In most real world applications, however, the data does not have a linear structure. In other words, the underlying distribution may not be well-represented by any low-rank subspace of the ambient space. In such settings, the representations learned using PCA may not be very informative. This motivates the need for non-linear dimensionality reduction methods. For example, in kernel PCA~\citep{scholkopf1998nonlinear}, a canonical approach for manifold learning, a nonlinear  \textit{feature map} lifts the data into a higher (potentially infinite) dimensional Reproducing Kernel Hilbert Space (RKHS), where a low-rank subspace corresponds to a (non-linear) low-dimensional manifold in ambient space. Hence, solving the PCA problem in an RKHS can better capture the complicated nonlinear structure in  data. 

Formally, given a kernel function $k(\cdot,\cdot):\mathbb{R}^d\times\mathbb{R}^d \rightarrow \mathbb{R}$, KPCA can be formulated as the following stochastic optimization problem:

\begin{equation}\label{prob:main}
    \maximize{\P}{\bE_{\x\sim\rho}{\langle\P,k(\x,\cdot)\otimes_\cH k(\x,\cdot)\rangle}}{\P \in \cP^k_{HS(\cH)}},
\end{equation}
where $\cP^k_{HS(\cH)}$ is the set of all orthogonal projection operators onto a $k$-dimensional subspace of the RKHS. The solution to the above problem is given by $\P_{\C}^k$, the projection operator corresponding to the top-k eigenfunctions of the covariance operator $\C:=\bE_{\x\sim \rho}[k(\x,\cdot)\otimes_\cH k(\x,\cdot)]$. The primary goal of any KPCA algorithm is then to guarantee generalization, i.e. providing a solution $\hat \P \in \cP_{HS(\cH)}^k$ with a small \textit{excess risk}:
\begin{equation}\label{eq:excess_risk}
    \cE(\hat{\P}) := \bE_{\x\sim\rho}{\langle\P_{\C}^k,k(\x,\cdot)\otimes_\cH k(\x,\cdot)\rangle} - \bE_{\x\sim\rho}{\langle \hat\P,k(\x,\cdot)\otimes_\cH k(\x,\cdot)\rangle}.
\end{equation}
Given access to i.i.d. samples $\{\x_i\}_{i=1}^n \sim \rho$, one approach to solving  Problem~\eqref{prob:main} is {\bf{Empirical Risk Minimization (ERM)}}, which amounts to finding the top-$k$ eigenfunctions of the empirical covariance operator $\hat\C:=\frac{1}{n}\sum_{i=1}^{n}{k(\x_i,\cdot)\otimes k(\x_i,\cdot)}$. Using kernel trick, \cite{scholkopf1998nonlinear} showed that this problem is equivalent of finding the top-$k$ eigenvectors of the kernel matrix associated with the samples.
Alternatively, when approximating the kernel map with random features, Problem~\eqref{prob:main} reduces to the PCA problem (given in Equation~\eqref{prob:pca}) in the random feature space. Here, we discuss two natural approaches to solve this problem. First, the ERM in the random feature space (called {\bf{RF-ERM}}), which is given by the top-$k$ eigenvectors of the empirical covariance matrix of data in the feature space. Second, the classical Oja's algorithm (called {\bf{RF-Oja}})~\citep{oja1982simplified}. 

Note that while the output of ERM is guaranteed to induce a projection operator in the RKHS of $k(\cdot,\cdot)$, 
this may not be the case when using RFF (equivalently, when working in the RKHS associated with the approximate kernel map). 
Therefore, a key technical challenge when designing KPCA algorithm based on RFF is to ensure that the output is close to the set of projection operators in the true RKHS induced by $k(\cdot, \cdot)$, i.e. $d(\hat\P, \cP_{HS(\cH)}^k)$ is small.

% \vspace*{-7pt}
\section{Mathematical Preliminaries and Notation}
\label{sec:prelim}
% \vspace*{-7pt}
In this section, we review basic concepts we need from functional analysis~\citep{reedmethods}. We begin with a simple observation that given an underlying distribution on data, and a fixed kernel map, it induces a distribution on the feature map. We work with this distribution implicitly by considering measurable Hilbert spaces. 
We denote matrices and Hilbert-Schmidt operators with capital roman letters $\D$, vectors with lower-case roman letters $\v$, and scalars with lower-case letters $a$. We denote operators over the space of Hilbert-Schmidt operators with capital Fraktur letters $\fA$.

\paragraph{Hilbert space notation and operator norm.} Let $\cH$ and $\tilde{\cH}$ be two separable Hilbert spaces over fields $\mathbb{F}$ and $\tilde{\mathbb{F}}$ with measures $\mu$ and $\tilde{\mu}$, respectively. Let $\{e_i\}_{i \geq 1}$ and $\{\tilde{e}_i\}_{i \geq 1}$ denote some fixed orthonormal basis for $\cH$ and $\tilde{\cH}$ respectively. The inner product between two elements $h_1, h_2 \in \cH$ is denoted as $\ip{h_1}{h_2}_\cH$, or $\ip{h_1}{h_2}_\mu$. Similarly, we denote the norm of an element $h \in \cH$ as $\norm{h}_{\cH}$, or $\norm{h}_{\mu}$.
For $h_1,h_2  \in \cH$ the outer product denoted as $h_1 \otimes_\cH h_2 $, or $h_1 \otimes_\mu h_2$, is a linear operator on $\cH$ that maps any $h_3 \in \cH$ to $(h_1 \otimes_\cH h_2)h_3 = \ip{h_2}{h_3}_\cH h_1$.
For a linear operator $\D:\cH\rightarrow \tilde \cH$, the \textit{operator norm} of $\D$ is defined as $\norm{\D}_2:=\sup \{\norm{\D h}_{\tilde{\cH}}, h \in \cH, \norm{h}_\cH \leq 1\}$.

\paragraph{Adjoint, Hilbert-Schmidt, and trace-class operators.}
The \textit{adjoint} of a linear operator $\D: \cH \rightarrow \tilde \cH$, is given as the linear operator $\D^* :  \tilde{\cH} \rightarrow \cH$ such that $\langle \D h,\tilde{h} \rangle_{\tilde{\cH}} = \langle h, \D^* \tilde{h}\rangle_{\cH}$, for all $ h \in \cH,\tilde{h} \in \tilde{\cH}$.
A linear operator $\D: \cH \rightarrow \cH$ is \textit{self-adjoint} if $\D^* =  \D$. 
The linear operator $\D: \cH \rightarrow \tilde{\cH}$ is \textit{compact} if the image of any bounded set of $\cH$ is a  relatively compact subset of $\tilde{\cH}$.
A linear operator $\D :\cH \rightarrow \cH$ is a \textit{Hilbert-Schmidt operator} if $\sum_{i\geq 1} \norm{\D e_i}_\cH^2 = \sum_{i,j \geq 1} \ip{\D e_i}{e_j}_\cH^2 < \infty$.
The \textit{Hilbert-Schmidt norm} of $\D$, denoted as $\norm{\D}_{HS(\cH)}$ or $\norm{\D}_{HS(\mu)}$, is defined as $({\sum_{i\geq 1} \norm{\D e_i}_\cH^2})^{\frac12}$. 
The space of all Hilbert-Schmidt operators on $\cH$ is denoted as $HS(\cH)$. 
A compact operator $\D:\cH \rightarrow \cH$ is \textit{trace-class} if $\norm{\D}_{\cL^1(\cH)}:=\sum_{i \geq 1}\ip{(\D \D^*)^{1/2} e_i}{e_i}_\cH < \infty$, where $\norm{\D}_{\cL^1(\cH)}$ denotes the nuclear norm of $\D$. 
For a vector space $\cX$, $L^2(\cX, \rho)$ denotes the space of square integrable functions with respect to measure $\rho$, i.e $L^2(\cX, \rho) = \{ f: \cX \rightarrow \R, \int_\cX (f(\x))^2 d \rho (\x) < \infty\}$. $L^2(\cX, \rho)$ is a Hilbert space with the inner product denoted as $\ip{f}{g}_\rho : = \int_\cX f(\x) g(\x) d \rho (\x)$, where $f, g \in L^2(\cX, \rho)$. The norm induced on $L^2(\cX, \rho)$ is denoted as $\norm{f}_\rho := \ip{f}{f}_\rho^{1/2}$ for $f \in L^2(\cX, \rho)$.  

\paragraph{Projection operators, spectral decomposition.}
Given a vector space $\cX$, let $\cP_\cX^k$ denote the set of rank-$k$ projection operators on $\cX$. For a Hilbert-Schmidt operator $\D$ over a separable Hilbert space $\cH$, let $\lambda_i(\D)$ denote its $i^\text{th}$ largest eigenvalue. The projection operator associated with the first $k$ eigenfunctions of $\D$ is denoted as $\P^k_\D$; given the spectral decomposition $\D = \sum_{i=1}^\infty \mu_i \psi_i \otimes \psi_i$, we have that $\P^k_\D = \sum_{i=1}^k \psi_i \otimes \psi_i$. For a finite dimensional vector $\v$, $\norm{\v}_p$ denotes the $\ell_p$-norm of $\v$. For operators $\D$ over finite dimensional spaces, $\norm{\D}_2$ and $\norm{\D}_F$ denote the spectral and Frobenius norm of $\D$, respectively. For a metric space $(\Y,d)$ and a closed subset $\S \subseteq \Y$, we denote the distance from $\q\in\Y$ to $\S$ by $d(q,S) = \min_{s\in S} d(q,s)$. In a Hilbert space, $d$ is the underlying metric induced by the respective norm. $[n]$ denotes the set of natural numbers from $1$ to $n$.

\paragraph{Mercer kernels, and random feature maps.} Let $\cX \subseteq \R^d$ be a compact (data) domain and $\rho$ be a distribution on $\cX$. We are given $n$ independent and identically distributed samples from $\rho$, $\{\x_i\}_{i=1}^n \sim \rho^n$. Let $k:\cX\times \cX \to \R$ be a Mercer kernel with the following integral representation, $k(\x,\y) = \int_\Omega z (\x,\omega) z (\y,\omega) d\pi(\omega)$. Here, $(\Omega,\pi)$ is the probability space induced by the Mercer kernel. 
Let $z_\omega(\cdot) := z(\cdot,\omega)$. We know that $z_\omega(\cdot)\in  L^2(\cX, \rho)$ almost surely with respect to $\pi$. We draw i.i.d. samples, $\omega_i \sim \pi,$ for $i =1, \ldots, m$, to approximate the kernel function. Let $\z(\cdot)$ denote the random feature map, i.e. $\z : \R^d \rightarrow \R^m, \z(\x) = \frac{1}{\sqrt{m}} \left(z_{\omega_1}(\x), z_{\omega_2}(\x), \ldots, z_{\omega_m}(\x)\right)$. Let $\cF \subseteq \R^m$ be the linear subspace spanned by the range of $\z$, with the inner product inherited from $\R^m$. The approximate kernel map is denoted as $k_m (\cdot,\cdot)$, where $k_m(\x,\y) = \ip{\z(\x)}{\z(\y)}_\cF$. Let $\cH$ denote the separable RKHS associated with the kernel function $k(\cdot, \cdot)$. 

% \vspace*{-2pt}
\begin{assumption}  \label{assumption:kernel}
The kernel function $k$ is a Mercer kernel (see Theorem~\ref{thm:mercer}) and has the following integral representation,
$k(\x,\y) = \int_\Omega z (\x,\omega) z (\y,\omega) d\pi(\omega) \ \forall \x,\y \in \cX$ where 
$\cH$ is a separable RKHS of real-valued functions on
$\cX$ with a bounded positive definite kernel $k$. We also assume that there exists $\tau >  1$ such that $|z(\x,\omega)|\leq \tau$ for all $\x \in \cX, \omega \in \Omega$.
\end{assumption}

% \vspace*{-2pt}
Note that $z (\x,\cdot)$ are continuous functions because $k(\cdot, \cdot)$ is continuous. 
Note that when $\cX$ is separable and $k(\cdot, \cdot)$ is continuous, $\cH$ is separable.
\begin{definition}\label{def:C}
$\C : \cH \rightarrow \cH$ is the covariance operator of the random variables $k(\x,\cdot)$ with measure $\rho$, defined as $\C f := \int_\cX k(\x,\cdot)f(\x)d \rho(\x)$. $\C$ is compact and self-adjoint, which implies $\C$ has a spectral decomposition $\C = \sum_{i =1}^\infty \bar{\lambda}_i \bar{\phi}_i \otimes_\cH \bar{\phi}_i,$
where $\bar{\lambda}_i$'s and $\bar{\phi}_i$'s are the eigenvalues and eigenfunctions of $\C$, respectively. The set of eigenfunctions,  $\{\bar{\phi}_i\}_{i=1}^\infty$, forms a unitary basis for $\cH$.
\end{definition}
Since we are approximating the kernel $k$ by sampling $m$ i.i.d.\ copies of $z_\omega$, this implies an approximation to the covariance operator $\C$ (in the space $HS(\rho)$) by a sample average of the random linear operators $z_\omega \otimes_\rho z_\omega$. The tools we use to establish concentration require a sufficient spectral decay of the variance of this random operator, which we define next.
\begin{definition}\label{def:C_2}
Let $\C_1$ denote the random linear operator on $L^2(\cX,\rho)$ given by $\C_1 = z_\omega \otimes_\rho z_\omega$. Let $\C_2 = \C_1 \otimes_{HS(\rho)} \C_1$ and define the covariance operator of $\C_1$ to be $\C' = \expectation{\pi}{\C_2} - \expectation{\pi}{\C_1} \otimes_{HS(\rho)} \expectation{\pi}{\C_1}$.
\end{definition}
We note that $\C'$ can also be interpreted as the fourth moment of the random variable $z_\omega$ in $L^2(\cX,\rho)$. The spectrum of $\C'$ plays a crucial role in our results through the following key-quantity: 

% \vspace*{-3pt}
\begin{equation}\label{eq:kappa}
\kappa(B_k,k,m) = \inf_{h\geq 0}\left\{\frac{B_k h}{m} + \sqrt{\frac{k}{m}\sum_{j>h}\lambda_i(\C')}\right\}, \quad \text{where} \ B_k := \frac{\sqrt{\expectation{\pi}{\langle z_\omega,z_\omega \rangle^4_\rho}}}{\bar\lambda_k - \bar\lambda_{k+1}}
\end{equation}
Essentially, we will see that the constant $\kappa(B_k,k,m)$ is the dominating factor when bounding the excess risk, and, therefore, will determine the rate of convergence of our algorithms. 

From a practical perspective, working in $HS(\rho)$ is not computationally feasible. However, our approximation to $\C$ has a representation in the finite dimensional space $\cF$, as defined here.

\begin{definition}\label{def:C_m}
$\C_m : \cF \rightarrow \cF$ is the covariance operator in $HS(\cF)$, defined as $\C_m := \expectation{\rho}{\z(\x) \otimes_\cF \z(\x)}$.
Equivalently, for any $\v \in \cF, \C_m \v = \int_\cX \ip{\z (\x)}{\v}\z(\x) d \rho (\x) $. 
$\C_m$ is compact and self-adjoint which implies that $\C_m$ has a spectral decomposition $\C_m = \sum_{i=1}^{m} \lambda_i \phi_i \otimes_\cF \phi_i$.
\end{definition}

As mentioned at the beginning of the section, our convergence tools work most conveniently when we can incorporate the randomness with respect to $\rho$ in the geometry of the space we study, hence, the need to study $L^2(\cX,\rho)$. Since we are essentially dealing with random operators on $\cF,\cH$ and $L^2(\cX,\rho)$, it is most appropriate to also work in the respective spaces of Hilbert-Schmidt operators. Thus, we introduce the \emph{inclusion} and \emph{approximation} operators, which allow us to transition with ease between the aforementioned spaces.

\begin{definition}\label{def:inclusion}[Inclusion Operators $\sI$ and $\fI$]\normalfont~
The inclusion operator is defined as 
% \vspace*{-5pt}
\[\sI : \cH \rightarrow L^2(\cX,\rho), \ (\sI f) = f, \text{ where } f \in \cH. \]

% \vspace*{-5pt}
Also, for an operator $\D \in HS(\cH)$ with spectral decomposition $\D = \sum_{i=1}^\infty \mu_i \psi_i \otimes \psi_i$,
% \vspace*{-7pt}

$$\fI: HS(\cH) \rightarrow HS(\rho), \ \fI \D:= \sum_{i = 1}^\infty \mu_i  \frac{\sI \psi_i}{\sqrt{\langle \C \psi_i, \psi_i\rangle_\cH}}  \otimes \frac{\sI \psi_i}{\sqrt{\langle \C \psi_i, \psi_i\rangle_\cH}}.$$
\end{definition}

In Lemma \ref{lemma:inclusion1} and Lemma \ref{lemma:inclusion2} in the appendix, we show that the adjoint of the Inclusion operator $\sI$ is $\sI^*: L^2(\cX,\rho) \rightarrow \cH$ given by $(\sI^*g)(\cdot) = \int k(\x,\cdot)g(\x) d \rho (\x)$, and that $\C = \sI^*\sI, \L = \sI \sI^*$. 

\begin{definition}\label{def:approximation}[Approximation Operators $\sA$ and $\fA$]
The Approximation operator $\A$ is defined as \\
\[\sA : \cF \rightarrow L^2(\X,\rho),  (\sA \v)(\cdot) = \ip{\z (\cdot)}{\v}, \text{ where } \v \in \cF. \]
For an operator $\D \in HS(\cF)$ with rank $k$ with spectral decomposition $\D = \sum_{i =1}^\infty \mu_i \psi_i \otimes \psi_i$, let $\Psi$ be the matrix with eigenvectors $\psi_i$ as columns and let $\Phi$ be the matrix with eigenvectors of $\C_m$ as columns (see Definition~\ref{def:C_m}). Define 
\vspace*{-5pt}
$$\RR^* = \argmin_{\RR^\top\RR = \RR \RR^\top = \I} \norm{\Psi \RR - \Phi}^2_\cF, \ \tilde{\Psi} := \Psi \RR^*.$$

Let $\tilde{\psi}_i$ be the $i^{th}$ column of $\tilde{\Psi}$, define
\vspace*{-10pt}

$$\fA: HS(\cF) \rightarrow HS(\rho), \ \fA \D:= \sum_{i=1}^k \mu_i  \frac{\sA \tilde{\psi}_i}{\sqrt{\langle \C_m\tilde{\psi}_i,\tilde{\psi}_i\rangle_\cF}}  \otimes_{\rho} \frac{\sA \tilde{\psi}_i}{\sqrt{\langle \C_m\tilde{\psi}_i,\tilde{\psi}_i\rangle_\cF}}.$$
\end{definition}

In Lemma \ref{lemma:approximation1} and Lemma \ref{lemma:approximation2}, we show that the adjoint of the Approximation Operator is $\sA^*: L^2(\X,\rho) \rightarrow \cF, (\sA^* f)_i = \int_\cX f(\x) z_{\omega_i} (\x) d\rho (\x)$, and $\C_m = \sA^*\sA, \L_m = \sA \sA^*$.

We note that the definition of the approximation operator $\fA$ requires knowledge of the covariance matrix $\C_m$ to find the optimal rotation matrix $\RR^*$, but this is solely for the purpose of analysis and is not used in the algorithm in any form. 

The following definition enables us to bound the excess risk in $HS(\cH)$ (Section~\ref{sec:app:eq} in the appendix).
\begin{definition}
\label{def:lift}
[Operator $\fL$]
Let $\tilde \P \in HS(\cF)$. Let $\fA \tilde \P = \sum_{i=1}^k \tilde \p_i \otimes_\rho \tilde \p_i$ be $\tilde \P$ lifted to $HS(\rho)$. Consider the equivalence relation  $\p_i \sim \p_j$ if $\L^{1/2} \p_i = \L^{1/2}  \p_j$. Let $[\p_i]$ be the equivalence class such that $\L^{1/2} \p_i = \tilde \p_i$.  The operator $\fL: HS(\cF) \rightarrow HS(\cH)$ is defined as $
\fL \hat \P = \sum_{i=1}^k \sI^* \p_i \otimes_\cH \sI^* \p_i$.
Here $\sI^*$ is the restriction of the operator $\sI^*$ to the quotient space $L^2(\cX,\rho)/\sim$.
\end{definition}

The quotient space in the definition above is with respect to the kernel of $\L$, i.e., $L^2(\cX,\rho)/\sim \equiv L^2(\cX,\rho)/\text{ker}(\L)$. 
This quotient is benign since the optimal solution to our optimization problem lives in the range of $\L$ and intuitively we can disregard any components in the kernel of $\L$.

\begin{wrapfigure}{r}{0.5\textwidth}
    \centering
    \vspace*{-11pt}
\begin{tikzpicture}[scale = 0.4]
\draw[blue,fill=blue!5] plot [smooth cycle] coordinates {(1.0,.1)(2.8,.5)(2.8,2.8)(1.4,2.5)} 
node[inner sep=0.8pt, label={[xshift=-0.6cm, yshift=-0.3cm]:{\footnotesize $\cX$}}] (X) at (1.8,1.8) {};
\draw plot [smooth cycle] coordinates {(0,-1.2)(0,-1.2)(0,-1.2)(0,-1.2)} 
node[inner sep=0.8pt,label={[xshift=-0.5cm, yshift=-.5cm]:{\footnotesize $\mathfrak{L}$}}](HSX) at (0,-1.2) {};
\draw[red,fill=red!5] plot [smooth cycle] coordinates {(-1.0,7.8)(0.8,8.2)(0.8,10.5)(-0.4,10.2)} 
node[inner sep=0.8pt, label={[xshift=-0.8cm, yshift=0.2cm]:{\scriptsize $HS(\cF)$}}] (HSF) at (0,9.1) {};    
\draw[blue,fill=blue!5]  plot [smooth cycle] coordinates {(1.0,5.1)(2.8,5.5)(2.8,7.8)(1.4,7.5)} 
node[inner sep=0.8pt, label={[xshift=-0.6cm, yshift=-0.7cm]:{\scriptsize $\cF$}}] (F) at (1.8,6.8) {};
\draw[red,fill=red!5] plot [smooth cycle] coordinates {(9.7,-2.25)   (11.3,-2.25) (11.3,0)  (9.5,-0.10)   } 
node[inner sep=0.8pt, label={[xshift=1.2cm, yshift=-0.6cm]:{\scriptsize $HS(\cH)$}}] (HSH) at (10.2,-1.2) {};     
\draw[blue,fill=blue!5] plot [smooth cycle] coordinates {(7,0.25)   (9,0.5) (9,2.65)  (7.8,2.75)   } 
node[inner sep=0.8pt, label={[xshift=0.8cm, yshift=-0.3cm]:{\footnotesize $\cH$}}] (H) at (8,1.8) {};
\draw[red,fill=red!5] plot [smooth cycle] coordinates {(9.2,8.25)   (11,8.5) (11,10.65)  (9.4,10.75)  } 
node[circle, fill=black, inner sep=0.8pt, label={[xshift=1.1cm, yshift=0.2cm]:
{\scriptsize $HS(\rho)$}}] (HSL2) at (10.0,9.2) {};  
\draw[blue,fill=blue!5] plot [smooth cycle] coordinates {(7,5.25)   (9,5.5) (9,7.65)  (7.8,7.75)  } 
node[inner sep=0.8pt, label={[xshift=1.1cm, yshift=-0.7cm]:{\scriptsize $L^2 (\cX,\rho)$}}] (L2) at (8,6.7) {}; 
\draw[-{Straight Barb[length=5pt,width=5pt]}] (F) edge[out=30, in=150] node[above] {{\footnotesize $\sA $}} (L2);
\draw[-{Straight Barb[length=5pt,width=5pt]},dashed] (L2) edge[out=210, in=-30] node[below] {{\footnotesize $\sA^* $}} (F);
\draw[-{Straight Barb[length=5pt,width=5pt]}] (HSF) edge[out=30, in=150] node[above] {{\footnotesize $\fA $}} (HSL2);
\draw[-{Straight Barb[length=5pt,width=5pt]}] (H) edge[out=60, in=-60] node[right] {{\footnotesize $\sI $}} (L2);
\draw[-{Straight Barb[length=5pt,width=5pt]},dashed] (L2) edge[out=-120, in=120] node[left] {{\footnotesize $\sI^* $}} (H);
\draw[-{Straight Barb[length=5pt,width=5pt]}] (HSH) edge[out=60, in=-60] node[right] {{\footnotesize $\fI $}} (HSL2);
\draw[-{Straight Barb[length=5pt,width=5pt]},dashed] (X) edge[out=120, in=-120] node[left] {{\footnotesize $\z $}} (F);
\draw[-{Straight Barb[length=5pt,width=5pt]},dashed] (X) edge[out=-30, in=210] node[above] {{\footnotesize$k(\x,\cdot) $}} (H);
\draw[-{ [length=5pt,width=5pt]},dashed] (HSF) edge[out=-120, in=120] node[above] {} (HSX);
\draw[-{Straight Barb[length=5pt,width=5pt]},dashed] (HSX) edge[out=-30, in=210] (HSH);
\draw[->,  very thick] (F) -- (HSF) ;
\draw[->,  very thick] (H) -- (HSH) ;
\draw[->,  very thick] (L2) -- (HSL2) ;
\end{tikzpicture} 
\vspace*{-15pt}
    \caption{Maps between the data domain ($\cX$), space of square integrable functions on $\cX$ ($L^2(\cX, \rho)$), the RKHS of kernel $k(\cdot, \cdot)$, and RKHS of the approximate feature map, as well as maps between Hilbert-Schmidt operators on these spaces.}
    \label{fig:maps-main}
\end{wrapfigure}

Finally, to conclude the section we give a visual schematic in Figure \ref{fig:maps-main} to help the reader connect different spaces. 
To summarize, the key spaces of interest are the data domain $\cX$, the RKHS $\cH$ of the kernel map $k(\cdot, \cdot)$, and the feature space $\cF$ obtained via random feature approximation. The space $L^2(\cX, \rho)$ consists of functions over the data domain $\cX$ that are square integrable with respect to the data distribution $\rho$. 
The space $L^2(\cX, \rho)$ allows us to embed objects from different spaces into a common space so as to compare them. Specifically, we map functions from $\cH$ to $L^2(\cX, \rho)$ via the  inclusion operator $\sI$, and vectors from $\cF$ to $L^2(\cX, \rho)$ via the approximation operator $\sA$. $\sI^*$ and $\sA^*$ denote the adjoints of $\sI$ and $\sA$, respectively. The space of Hilbert-Schmidt operators on $\cH, \cF$ and $L^2(\cX,\rho)$, are denoted by $HS(\cH)$, $HS(\cF)$ and $HS(\rho)$, respectively. 
Analogous to $\sI$ and $\sA$, $\fI$ maps operators from $HS(\cH)$ to $HS(\rho)$, and $\fA$  maps operators from $HS(\cF)$ to $HS(\rho)$, respectively.
Specifically, these are essentially constructed by mapping eigenvectors of operators via $\sI$ and $\sA$ respectively. 
The above mappings thus allow us to embed operators in the common space, i.e., $HS(\rho)$ and to bound estimation and approximation errors.
However, the problem of Kernel PCA is formulated in $HS(\cH)$ and bounds in $HS(\rho)$ are therefore not sufficient.
To this end, we establish an equivalence between kernel PCA in $HS(\cH)$ and $HS(\rho)$.
We use the map $\fA$ and the established equivalence to get $\fL$, which maps operators from $HS(\cF)$ to $HS(\cH)$. We encourage the reader to go through Sections~\ref{sec:app:prelim} and \ref{sec:app:eq} in the appendix for a gentler and a more rigorous presentation.

\section{Main Results}\label{sec:main_results}
\vspace*{-5pt}
Recall that our primary goal is to study the generalization behaviour of algorithms solving KPCA using random features. 
Rather than stick to a particular algorithm, we define a class of algorithms that are suitable to the problem. We characterize this class as follows. 

\begin{definition}[Efficient Subspace Learner (ESL)] \label{def:efficient-subspace-learner-main}
\normalfont
Let $\cA$ be an algorithm which takes as input $n$ points from $\cF$ and outputs a rank-k projection matrix over $\cF$. Let $\hat{\P}_\cA$ denote the output of the algorithm $\cA$ and $\hat{\P}_\cA = \tilde{\Phi}\tilde{\Phi}^\top$ be an eigendecompostion of $\hat{\P}_\cA$. Let $ \Phi^\perp_k $ be an orthogonal matrix corresponding to the orthogonal complement of the top $k$ eigenvectors of $\C_m$.
We say that algorithm $\cA$ is an Efficient Subspace Learner if the following holds with probability at least $1- \delta$,
\vspace*{-7pt}

$$\norm{ (\Phi^\perp_k)^\top \tilde{\Phi}}_F^2 \leq \frac{q_\cA^{\rho,\pi}(1/\delta, \log m, \log n)}{n},$$
where $q_{\cA}^{\rho,\pi}$ is a function given the triple $(\cA,\rho,\pi)$ which has polynomial dependence on $1/\delta$, $\log m$ and $\log n$. For notational convenience, we drop superscripts from $q_{\cA}^{\rho,\pi}$ and write it as $q_\cA$ henceforth.
\end{definition} 

Intuitively, an ESL is an algorithm which returns a projection onto a $k$-dimensional subspace such that the angle between the subspace and the space spanned by the top $k$ eigenvectors of $\C_m$ decays at a sub-linear rate with the number of samples. Our guarantees are in terms of any algorithm which belongs to this class. Algorithm \ref{alg:skpca} gives a high-level view of the algorithmic routine. To discuss the associated computational aspects, we instantiate this with two specific algorithms, ERM and Oja's algorithm, and show how the result looks in terms of their algorithmic parameters. Similar results can be obtained for other ESL algorithms such as $\ell_2$-RMSG~\citep{mianjy2018stochastic}. We now give the main theorem of the paper which characterizes the excess risk of an {ESL}. 

\begin{algorithm}[t]
\caption{KPCA with Random Features (Meta Algorithm)}
\label{alg:skpca}
\begin{algorithmic}[1]
\REQUIRE{Training data $\X = \{\x_i \}_{i=1}^n$}
\vspace{1pt}
\ENSURE{$\hat{\P}_\cA$}
\vspace{3pt}
\STATE Obtain Training data $\X = \{\x_t \}_{i=1}^n$ in a batch or stream \\
\STATE Sample $\omega_i \sim \pi$, i.i.d, $i = 1$ to $m$
\STATE $\Z \gets \texttt{RandomFeatures}(\X,\{\omega_i\}_{i=1}^m)$ \\
\STATE $\hat{\P}_\cA \gets \cA(\Z)$ \COMMENT{$\cA$ is an Efficient Subspace Learner, Definition \ref{def:efficient-subspace-learner-main}}
\vspace{4pt}
\end{algorithmic}
\end{algorithm}

\begin{theorem}[Main Theorem]
\label{thm:main_th}
Let $\cA$ be an efficient subspace learner. Let $\hat\P_\cA$ be the output of $\cA$ run with $m$ random features on $n \geq \frac{2 \lambda_1^2 \fqq{\cA}{2}^2}{\lambda_k^2 (\sqrt{2}-1)}$ points, where $\lambda_i$ is the $i^\text{th}$ eigenvalue of $\C_m$. Then, with probability at least $1-\delta$ it holds that
\vspace*{-5pt}
\begin{enumerate}
\item [(a).] $\cE(\fL \hat{\P}_\cA) \leq 24\kappa(B_k,k,m) + \frac{\log {\delta/2} + 7B_k}{m} + \sqrt{\frac{\fqq{\cA}{2}}{n}}$,
\item [(b).] $d\left(\fL \hat\P_\cA, \cP^k_{HS(\cH)}\right)\leq \sqrt{\frac{\fqq{\cA}{2}}{n}}$. 
\end{enumerate}

\end{theorem}
A few remarks are in order. First, as we forewarned the reader in Section~\ref{sec:prelim}, the error bound is dominated by the additive term $\kappa(B_k, k, m)$. This, in a sense,  determines the hardness of the problem. As we will see, under appropriate assumptions on data distribution in the feature space, this term can be bounded by something that is in $O(1/m)$. Second, the output of our algorithm, $\fL\hat\P$, need not be a projection operator in the RKHS. This is precisely why we need to bound the difference between $\fL\hat\P$ and the set of all projection operators in $HS(\cH)$, which we see is of the order $O(1/\sqrt{n})$.
Third, note that the dependence on the number of random features is at worst poly-logarithmic. 
From part $(b)$ of Theorem \ref{thm:main_th}, it is easy to see that if we project $\fL \hat \P_\cA$ to the set of rank $k$ projection operators in $HS(\cH)$, we get the same rate of convergence. This is presented as Corollary \ref{cor:main_th_cor_appendix} in the appendix.

Next, we characterize ``easy'' instances of KPCA problems under which we are guaranteed a fast~rate. 
Specifically, we show that if the decay of the spectrum of the fourth order moment, $\C'$, of $\z_\omega$, is exponential, then the dominating factor, $\kappa(B_k,k,m)$ is in $O(1/m)$. Then, optimizing the number of random features w.r.t. the sample complexity term gives us the following result. 

\begin{corollary}[Main - Good decay]
\label{cor:main_good_decay}
Along with the assumptions and notation of Theorem \ref{thm:main_th}, if the spectrum of the operator $\C'$ has an exponential decay, i.e., $\lambda_j(\C') = \alpha^{j}$ for some $\alpha < 1$, then with $m = O(\sqrt{n} \log n)$ random features, we have
\begin{align*}
    \cE(\fL \hat{\P}_\cA) \leq \frac{c B_k}{\sqrt{n}} + \frac{c'(k +\log {\delta/2} + 7B_k)}{\sqrt{n} \log n} + \sqrt{\frac{\fqq{\cA}{2}}{n}},
\end{align*}
where $c$ and $c'$ are universal constants.
\end{corollary}

Finally, we instantiate the above corollary with two algorithms, namely ERM and Oja's algorithm.

\begin{corollary}[ERM and Oja]\label{cor:main_erm_oja}
With the same assumptions and notation as in Corollary \ref{cor:main_good_decay},
\vspace*{-6pt}
\begin{enumerate}
\item  [(a).] \textbf{RF-ERM} is an ESL with $\fq{ERM} =\frac{k\lambda_1\tau^2}{\gap^2}\log{\frac{\delta}{2m}}^2.$
\vspace*{-7pt}
\item  [(b).] \textbf{RF-Oja} is an ESL with $\fq{oja} = \tilde \Theta \left(\frac{\Lambda}{\gap^2}\right)$, where $\Lambda = \sum_{i=1}^k \lambda_i$.
\vspace*{-5pt}
\end{enumerate}
where $\gap := \lambda_k(\C_m) - \lambda_{k+1}(\C_m)$.
\end{corollary}

\paragraph{Error Decomposition:} 
There are two sources of error when solving KPCA using random features -- the estimation error ($\epsilon_e$) resulting from the fact that we have access to the distribution only through an i.i.d. sample,  and approximation error ($\epsilon_a$) resulting from the approximate feature map. Therefore, to get a better handle on the excess error, we decompose it as follows.  

\vspace*{-15pt}
\begin{align*}
\cE(\fL \hat{\P}_\cA) = \underbrace{\langle \P^k_\C , \C \rangle_{HS(\cH)} - \langle \fL \P^k_{\C_m},  \C \rangle_{HS(\cH)}}_{\epsilon_a: \text{ Approximation Error}} + \underbrace{\langle \fL \P^k_{\C_m},  \C \rangle_{HS(\cH)} -  \langle \fL \hat{\P}_\cA, \C \rangle_{HS(\cH)}}_{\epsilon_e: \text{ Estimation Error}}.
\end{align*}

The main idea behind controlling the approximation error is to interpret it as the error incurred in eigenspace estimation in $L^2(\cX, \rho)$, and then use local Rademacher complexity to get faster rates. In the context of Kernel PCA, this technique was first used by \cite{blanchard2007statistical} which allowed them to get sharper $O(1/n)$ excess risk. The estimation error is controlled by the definition of our lifting map $\fA$ together with the convergence rate implicit in the definition of an {ESL}. Below, we guide the reader through the main steps taken to bound each of the error terms.

\vspace*{-8pt}
\paragraph{Bounding the Approximation error:} 
Using simple algebraic manipulations, we can show that the approximation error is exactly the error incurred by the ERM in estimating the top $k$ eigenfunctions of the kernel integral operator $\L$ using $m$ samples drawn from $\pi$. This problem of eigenspace estimation is well studied in the literature and has optimal statistical rates of $O\left(1/\sqrt{m} \right)$ \citep{zwald2006convergence}. This appears to be a key bottleneck and reinforces the view that the use of random features cannot provide computational benefits -- it suggests $m\!=\!\Omega(n)$ random features are required to get a $O\left(1/{\sqrt{n}}\right)$ rate. However, these rates are conservative when viewed in the sense of excess risk. This has been extensively studied in empirical process theory and one of the primary techniques to get sharper rates is the  use of local Rademacher complexity \citep{bartlett2002localized}. The key idea is to show that around the best hypothesis in the class, variance of the empirical process is bounded by a constant times the mean of the difference from the best hypothesis (see Theorem \ref{thm:localrademacher}). This technique was used in the context of Kernel PCA by \cite{blanchard2007statistical} to get fast $O(1/m)$ rates. We now state Lemma \ref{lemma:approxError} which bounds the approximation error, the proof of which is deferred to appendix.

\begin{lemma}[Approximation Error]
\label{lemma:approxError}
With probability at least $1 - \delta$, we have
\begin{align*} 
   \epsilon_a \leq 24 \kappa (B_k, k,m) + \frac{11  \tau^2 \log{\delta}  + 7B_k}{m}.
\end{align*}
\end{lemma}
\paragraph{Bounding the Estimation error:}
Since the objective with respect to the inner product in $HS(\rho)$ equals the objective with respect to the inner product in $HS(\cH)$ (See Lemma ~\ref{prop:eq4}), we focus on bounding the estimation error in $L^2(\cX,\rho)$. 
Using a Cauchy-Schwartz type of inequality in $HS(\rho)$, we see that it is enough to bound the difference $\| \fA \P^k_{\C_m} - \fA \hat{\P}_\cA \|_{HS(\rho)}$. We can do this in two steps -- bound the error $\|\P_{\C_m}^k - \hat\P_{\cA}\|_F$ (we already have this from the ESL guarantee) and construct $\fA : HS(\cF) \rightarrow HS(\rho)$. We already have a lifting from $\cF$ to $L^2(\cX,\rho)$ in the form of $\A$. The natural attempt to lift an operator on $\cF$ would be by lifting and appropriately rescaling its eigenfunctions. Since the eigendecomposition of $\hat\P_{\cA}$ is not unique, we need to choose an appropriate one to be lifted. Since the goal of $\fA$ is to preserve distances between operators, we choose the unique eigen-decomposition for which the distance $\sum_{i=1}^k \|\U_i - \phi_i \|_2^2$ is minimized. 
Notice that the lifting operator $\fA$ depends on the eigendecomposition of $\C_m$, which can not be obtained in practice. This is not a problem, because $\fA$ is only used for the purposes of showing the main result and is not part of the proposed algorithms. We now state Lemma \ref{lemma:estimationError} which bounds the estimation error.

\begin{lemma}[Estimation Error] 
\label{lemma:estimationError}
With the same assumptions as Theorem \ref{thm:main_th}, the following holds with probability at least $1 - \delta$,

\vspace*{-12pt}
    $$\epsilon_{e} \leq \frac{\lambda_1^2}{(\sqrt{2}-1)} \sqrt{\sum_{i=1}^k  \left(\frac{2 \lambda_i + 4 \lambda_1}{\lambda_i^2}\right)^2} \frac{\fq{\cA}^2}{n}.$$
\end{lemma}

\vspace*{-10pt}
\begin{figure*}[t]
\centering
\begin{tabular}{ccc}
$k=5$ & $k=10$ & $k=15$ \\
\hspace*{-10pt} 
\includegraphics[width=0.36\textwidth]{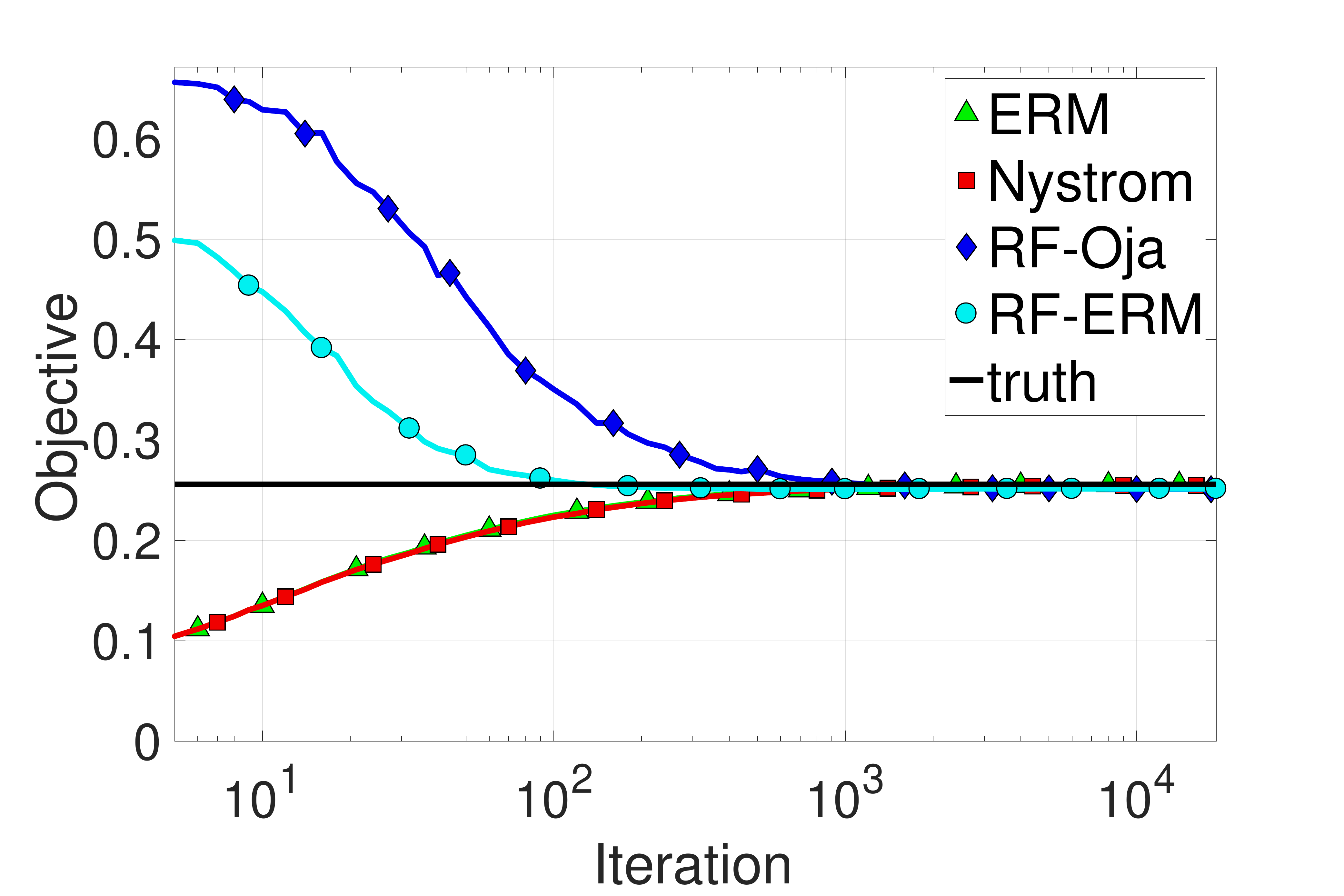}
&
\hspace*{-25pt} 
\includegraphics[width=0.36\textwidth]{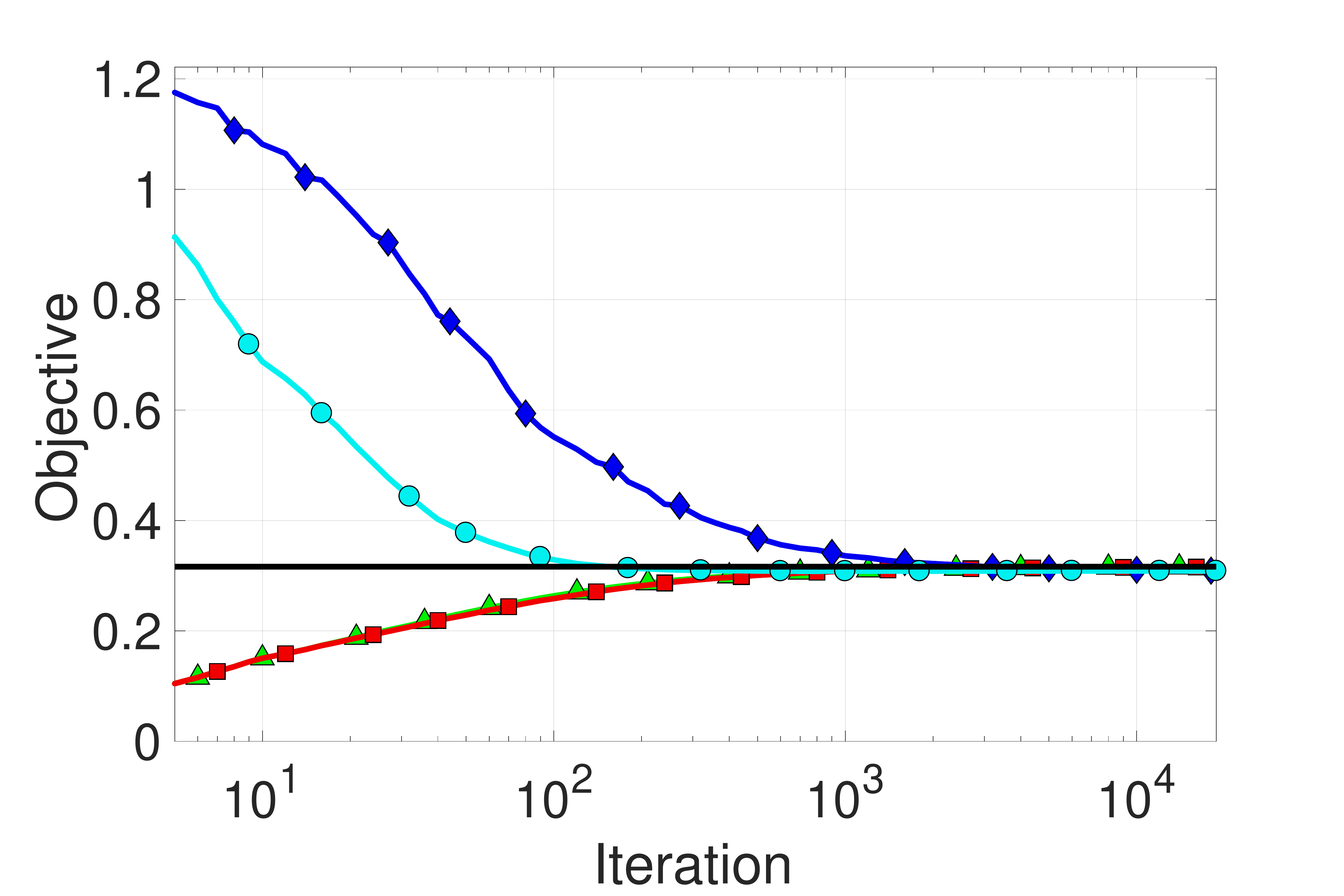}
&
\hspace*{-25pt} 
\includegraphics[width=0.36\textwidth]{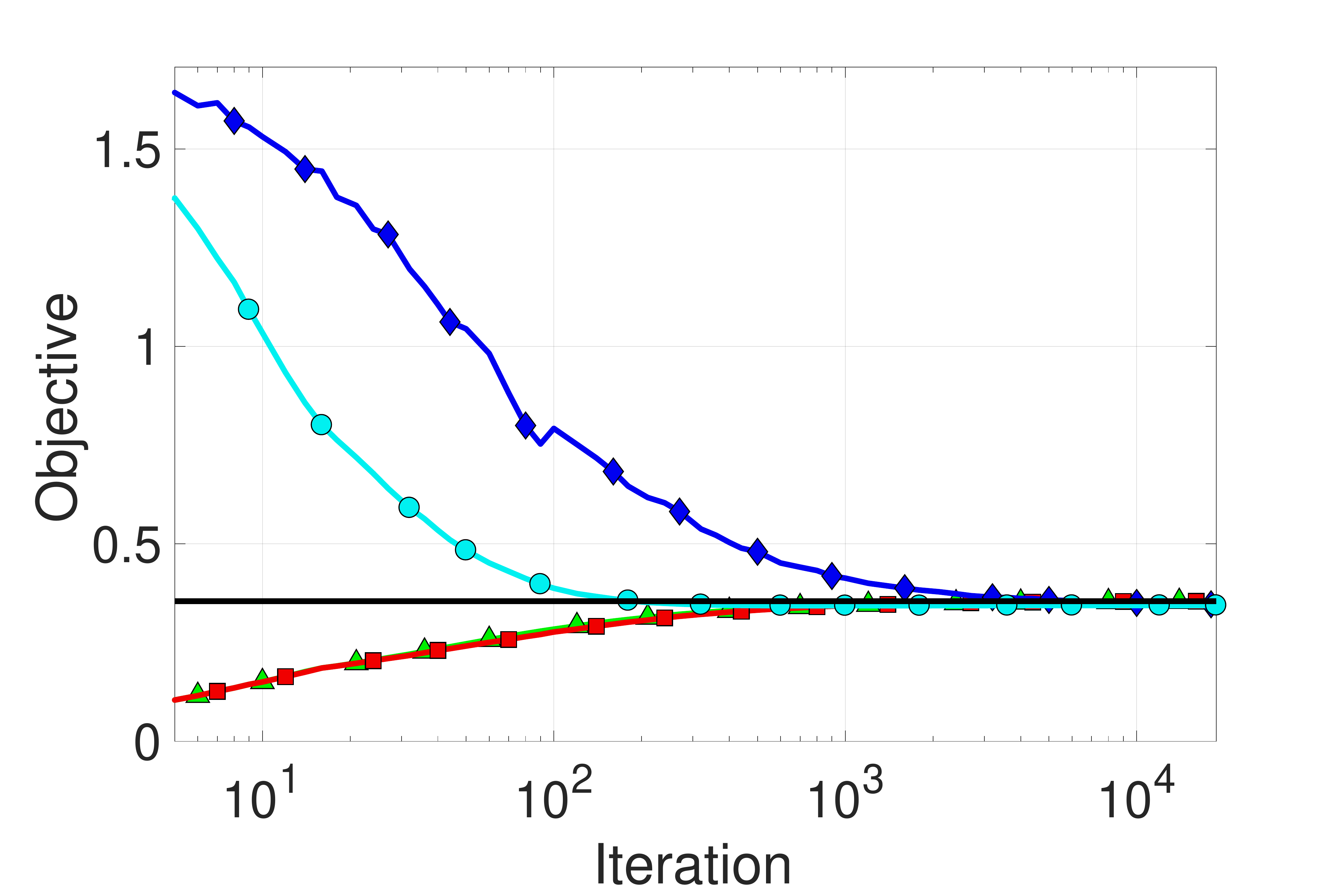}
\\
\hspace*{-10pt} 
\includegraphics[width=0.36\textwidth]{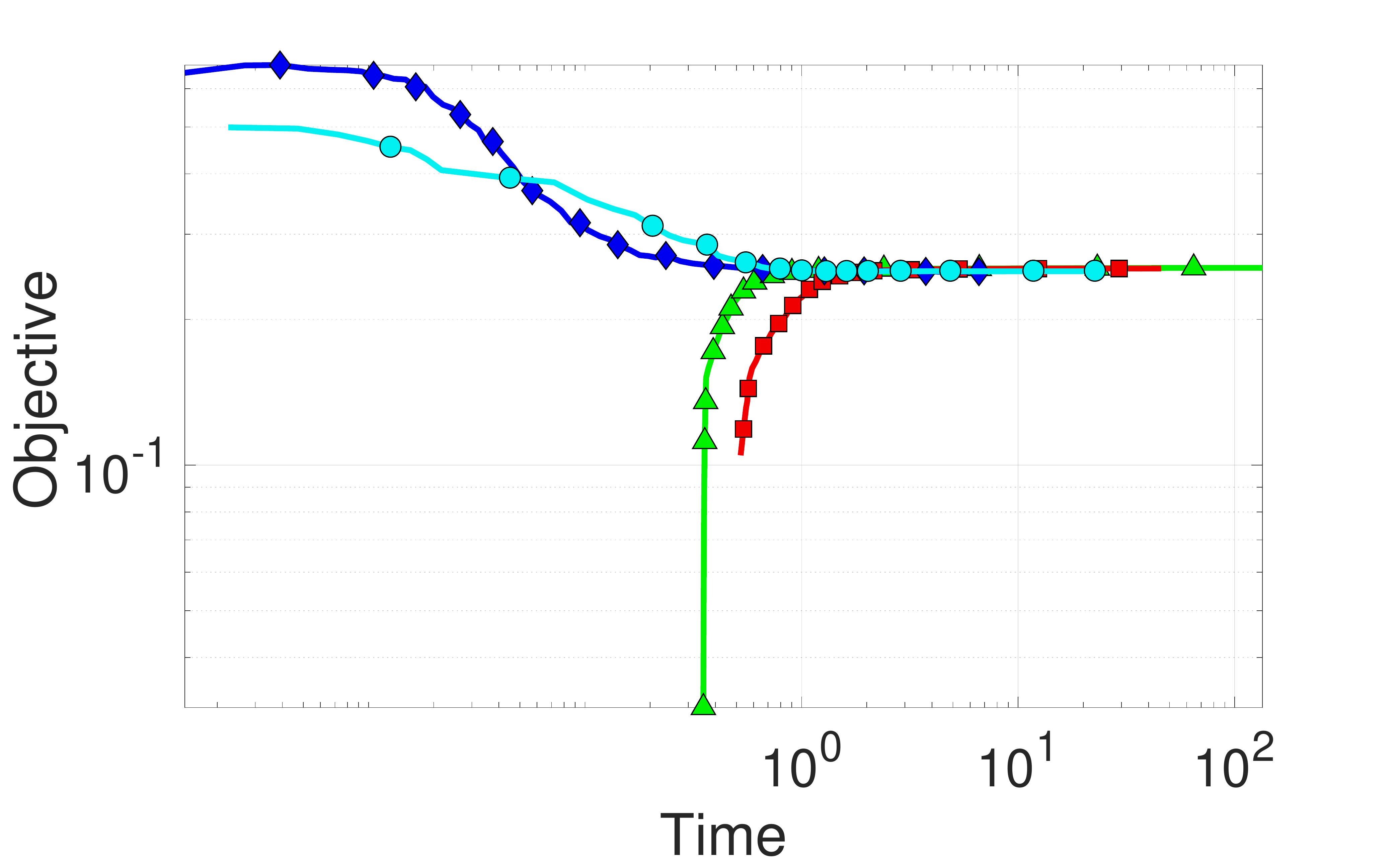}
&
\hspace*{-25pt} 
\includegraphics[width=0.36\textwidth]{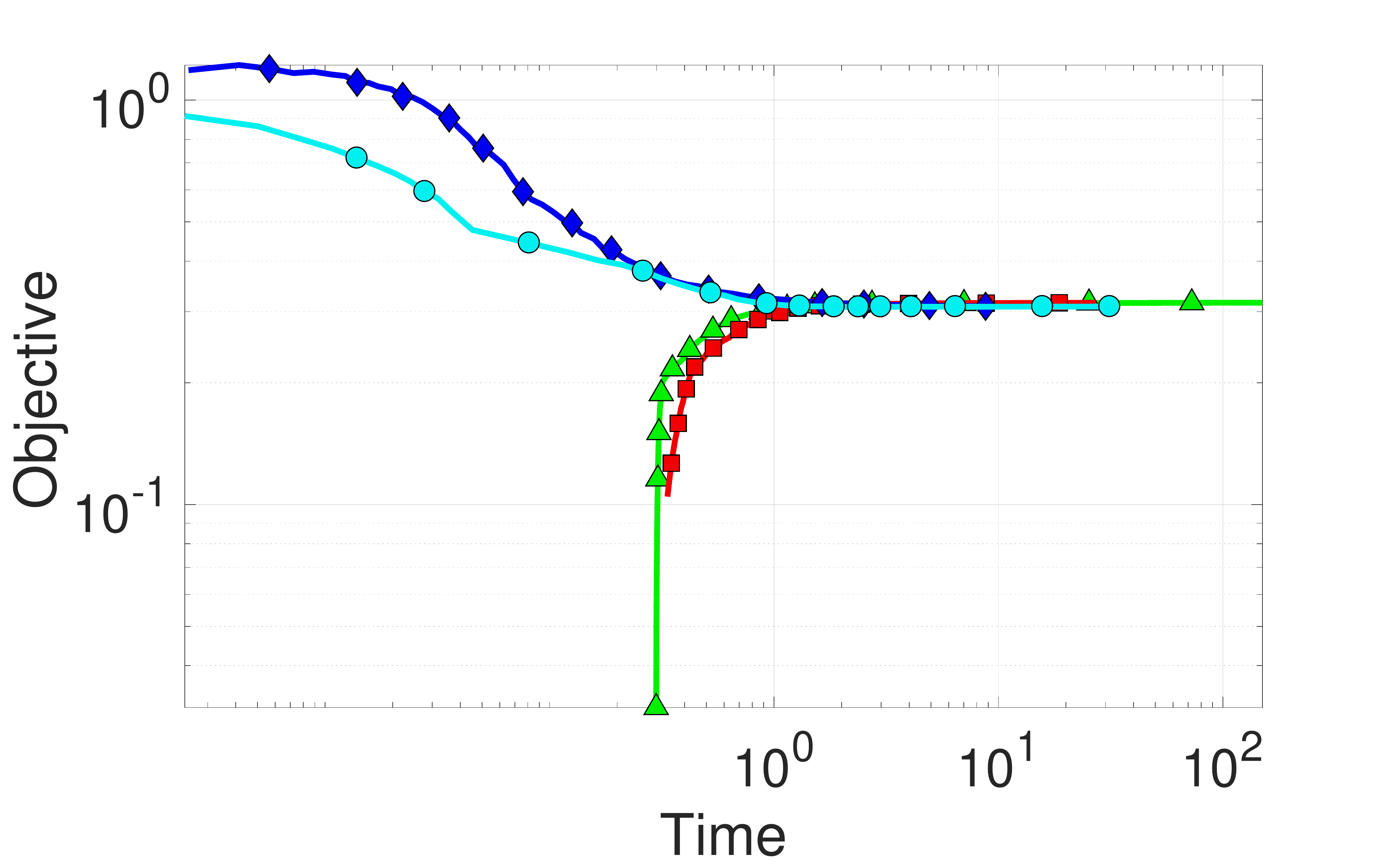}
&
\hspace*{-25pt} 
\includegraphics[width=0.34\textwidth]{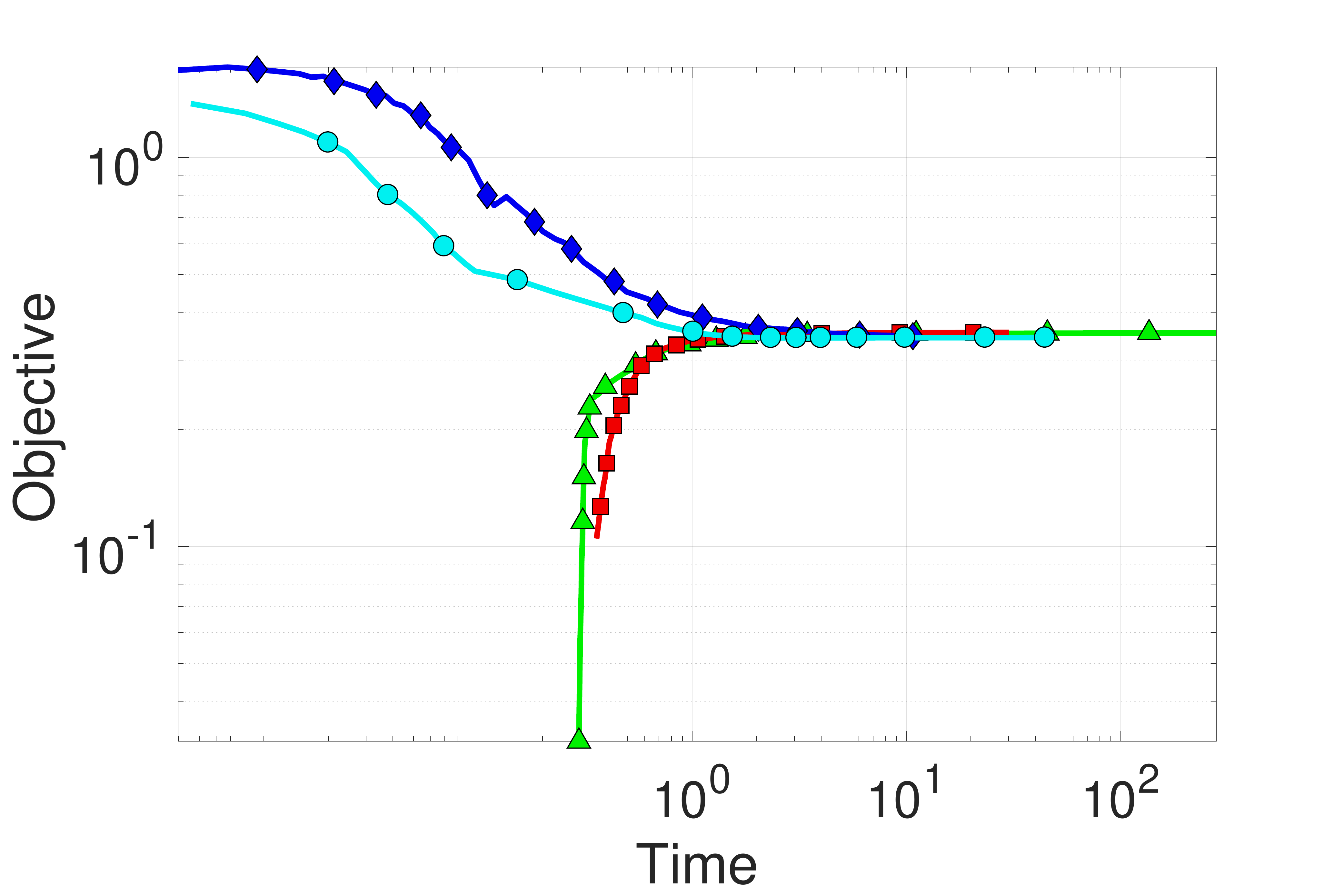}
\end{tabular}
\vspace*{-7pt}
\caption{\label{fig:real}\small{Comparisons of ERM, Nystr\"om, Oja+RFF, and Oja+ERM for KPCA on the MNIST dataset, in terms of the objective value as a function of iterations (top) and as a function of CPU runtime (bottom).}}
\label{fig:obj_mnist}
\end{figure*}

\section{Experiments}
\label{sec:experiments}
The goal of this section is to provide empirical evidence supporting our theoretical findings in Section~\ref{sec:main_results}. As we motivated in Section~\ref{sec:prob_form}, the success of an algorithm is measured in terms of it's \textit{generalization} ability, i.e. the variance captured by the output of the algorithm on the unseen data\footnote{Details on how we evaluate objective for RF-ERM/Oja are deferred to Section~\ref{sec:app:exp} due to space limitations.}. 

We perform experiments on the MNIST dataset that consists of $70$K samples, partitioned into a training, tuning, and a test set of sizes $20$K, $10$K, and $40$K, respectively. We use a fixed kernel in all our experiments, since we are not concerned about model selection here. In particular, we choose the RBF kernel $k(\x,\x')\!=\!\exp{-\|\x-\x'\|^2/2\sigma^2}$ with bandwidth parameter $\sigma^2= 50$. The bandwidth is chosen such that ERM converges in objective within observing few thousands training samples. The objective of the ERM\footnote{The kernel matrix is computed in an online fashion for computational efficiency} is used as the baseline. Furthermore, to evaluate the computational speedup gained by using random features, we compare against Nystr\"om method~\citep{drineas2005nystrom} as a secondary baseline. In particular, upon receiving a new sample, we do a full Nystr\"om approximation and ERM on the set of samples observed so far. Finally, empirical risk minimization (RF-ERM) and Oja's algorithm (RF-Oja) are used with random features to verify the theoretical results presented in Corollary~\ref{cor:main_erm_oja}.  

Figure~\ref{fig:obj_mnist} shows the population objective as a function of iteration (top row) as well as the total runtime\footnote{Runtime is recorded in a controlled environment; each run executed on identical unloaded compute node.} (bottom row). Each curve represents an average over $100$ runs of the corresponding algorithm on training samples drawn independently and uniformly at random from the whole dataset.
Number of random features and the size of Nystr\"om approximation are set to  $750$ and $100$, respectively.
We~note:  

\vspace*{-7pt}
\begin{itemize}
\item As predicted by Corollary~\ref{cor:main_erm_oja}, for both RF-ERM and RF-Oja, $\sqrt{n}\log n \approx 750$ random features is sufficient to achieve the same suboptimality as that of ERM.
\item 
The performance of ERM is similar to that of RF-Oja and RF-ERM in terms to overall runtime. However, due to larger space complexity of $O(n^2)$, ERM becomes infeasible for large-scale problems; this makes a case for streaming/stochastic approximation algorithms.
\end{itemize}
Finally, we note that the iterates of RF-ERM and RF-Oja \textit{reduce} the objective as they approach from above to the \textit{maximizer} of the population objective. 
Although it might seem counter-intuitive, we note that the output of RF-ERM and RF-Oja are not necessarily projection operators. Hence, they can achieve higher objective than the maximum. However, as guaranteed by Corollary~\ref{cor:main_erm_oja}, the output of both algorithms will converge to a projection operator as more training samples are introduced.

\newpage

\section*{Acknowledgements}
\noindent This research was supported in part by NSF BIGDATA grant IIS-1546482. 

\bibliographystyle{plainnat}   
\bibliography{main}

\newpage
\appendix
The Appendix is divided into the following six sections,
\begin{enumerate}
    \item [A.]\textit{Preliminaries and Structural Results}
    \item [B.]\textit{Equivalence of Optimization problem in $\cH$ and $L^2(\cX, \rho)$}
    \item [C.]\textit{Proof of the main Theorem}
    \item [D.]\textit{Examples of ESL}
    \item [E.]\textit{Experiments}
    \item [F.]\textit{Auxiliary Results}
\end{enumerate}

\section{Preliminaries and Structural Results}\label{sec:app:prelim}
We begin this section by giving some definitions and structural theorems that are required for the proofs. We first introduce a few additional notation. Consider two separable Hilbert spaces $\cH_1$ and $\cH_2$. For an operator $\cD:\cH_1 \rightarrow \cH_2$, we use $\norm{\cD}_{\cL^p(\cH_1,\cH_2)}$ to denote its $p^{\text{th}}$ schatten norm, assuming that it is finite. We omit the $p$ when we talk about about Hilbert-Schmidt norm i.e. $p=2$.

\subsection{Covariance operators \texorpdfstring{$\C$}{} and \texorpdfstring{$\C_m$}{}}
The covariance operators in the RKHS and the random feature space are defined as follows:

\begin{definition}\label{def:app:C}
$\C : \cH \rightarrow \cH$ is the co-variance operator of the random variables $k(\x,\cdot)$ with measure $\rho$, defined as:
\[\C f := \int_\cX k(\x,\cdot)f(\x,t)d \rho(\x,t) \]
$\C$ is compact and self-adjoint, which implies $\C$ has a spectral decomposition as follows:
\[\C = \sum_{i =1}^\infty \bar{\lambda}_i \bar{\phi}_i \otimes_\cH \bar{\phi}_i \]
where $\bar{\lambda_i}, \bar{\phi}_i$'s are the eigenvalues and eigenfunctions of $\C$. Also,  $\bar{\phi_i}$ are a unitary basis for $\cH$.
\end{definition}

\begin{definition}\label{def:app:C_m}
$\C_m : \cF \rightarrow \cF$ is the covariance operator in the random feature space, defined as
$$\C_m := \expectation{\rho}{\z(\x,t) \otimes_\cF \z(\x,t)}$$
Equivalently, for any $\beta \in \cF, \C_m \beta = \int_\cX \ip{\z (\x,t)}{\beta}\z(\x,t) d \rho (\x,t) $. \\

$\C_m$ is compact and self-adjoint which implies that $\C_m$ has a spectral decomposition as follows:
$$\C_m = \sum_{i=1}^{m} \lambda_i \phi_i \otimes_\cF \phi_i $$
\end{definition}

The kernel integral operators and its approximation based on random features are defined as follows:

\begin{definition}\label{def:L}
The kernel integral operator $\L: L^2(\cX, \rho) \rightarrow L^2(\cX,\rho)$ is defined as follows:
\[\L g = \int_\cX k(\x,\cdot)g(\x)d \rho(\x) \ \forall \ g \in L^2(\cX,\rho)\]
\end{definition}

\begin{definition}\label{def:L_m}
$\L_m: L^2(\cX,\rho) \rightarrow L^2(\cX,\rho)$ is the (approximated) kernel integral operator, defined as:
$$(\L_m g)(\cdot) = \int_{\cX} k_m(\x,\cdot) g(\x) d \rho (\x)  $$
\end{definition}

We state the classical Mercer's and Bochner's theorems for completeness.
\begin{theorem}[Mercer's Theorem]\label{thm:mercer} For every positive definition kernel $k(\cdot , \cdot): \cX \times \cX \rightarrow \R$, there exits a set $\Omega$ with measure $\pi$, and functions $\varpi(\cdot): \cX \times \Omega \rightarrow \R$ such that the kernel has an integral representation of the following form,
\[k(\x,\y) = \int_\Omega z (\x,\omega) z (\y,\omega) d\pi(\omega) \ \forall \ \x,\y \in \cX\]
\end{theorem}

In particular, for shift-invariant kernels, we have

\begin{theorem}[Bochner's Theorem~\cite{rudin2017fourier}]\label{thm:bochner}
A continuous, real-valued, symmetric shift-invariant kernel $k : \cX \times \cX \rightarrow \R$ is a positive-definite kernel if and only if there exits a non-negative measure $\pi (\omega)$ such that $k(x - y) = \int_\cX \e^{i \omega^\top (x-y)} d \pi(\omega)$ i.e the inverse Fourier transform of $k(x-y)$
\end{theorem}

For a comprehensive list of kernels with their Fourier transform see Table~1 of \cite{xie2015scale}.

We now define operators to lift functions and operators from and into different spaces. These are crucially used in the analysis of the algorithm. See Figure \ref{fig:maps} in for a schematic of the lifting operators.

\subsection{Inclusions operators \texorpdfstring{$\sI, \fI$}{}}
We first recall the definitions of Inclusion operators $\sI, \fI$.

\begin{definition}\label{def:app:inclusion}[Inclusion Operators $\sI$ and $\fI$]
The inclusion operator is defined \\
\[\sI : \cH \rightarrow L^2(\cX,\rho), \ (\sI f) = f, \text{ where } f \in \cH \]

Also, for an operator $\D \in HS(\cH)$ with spectral decomposition $\D = \sum_{i \in I \subset \R} \mu_i \psi_i \otimes \psi_i$,
$$\fI: HS(\cH) \rightarrow HS(\rho), \ \fI \D:= \sum_{i \in I \subset \R} \mu_i  \frac{\sI \psi_i}{\sqrt{\langle \C \psi_i, \psi_i\rangle_\cH}}  \otimes \frac{\sI \psi_i}{\sqrt{\langle \C \psi_i, \psi_i\rangle_\cH}}$$
\end{definition}

In Proposition \ref{lemma:inclusion1}, we show that the adjoint of the Inclusion operator $\sI$ is $$\sI^*: L^2(\cX,\rho) \rightarrow \cH, (\sI^*g)(\cdot) = \int k(\x,\cdot)g(\x) d \rho (\x).$$ Moreover, In Proposition \ref{lemma:inclusion2} we show that the covariance operator and the kernel integral operator can be expressed in terms of $\I$ and $\I*$ as $\C = \sI^*\sI$ and $\L = \sI \sI^*$.

\begin{prop}\label{lemma:inclusion1} The following holds with regard to the inclusion operator,
\begin{enumerate}
    \item [(a).] The adjoint of the Inclusion operator $\sI$ is given by $(\sI^* g)(\cdot) = \int_\cX k(\x,\cdot) g(\x) d \rho(\x)$.
    \item [(b).] $\sI$ and $\sI^*$ are Hilbert-Schmidt.
\end{enumerate}
\end{prop}
\begin{proof}[Proof of Proposition~\ref{lemma:inclusion1}]
\textit{(a).} We first show that the adjoint of the Inclusion operator $\sI$ is given by $(\sI^* g)(\cdot) = \int_\cX k(\x,\cdot) g(\x) d \rho(\x)$. For $f \in \cH$ and $g \in L^2(\cX,\rho)$, we have that
\begin{align*}
    \ip{\sI f}{g}_\rho & = \ip{f}{g}_\rho \tag{Definition of $\I$} \\
    & = \int_\cX f(\x) g(\x) d \rho(\x) \\
    & = \int_\cX \ip{k(\x,\cdot)}{f}_\cH g(\x) d \rho(\x) \tag{Reproducing property}\\
    & = \int_{\cX}\ip{k(\x,\cdot)g(\x)}{f}_\cH  d \rho(\x) \tag{Linearity of inner product} \\
    & = \ip{\int_{\cX} k(\x,\cdot)g(\x)  d \rho(\x)}{f}_\cH \tag{Fubini's Theorem}\\
    & = \ip{\sI^* g}{f}_\cH
\end{align*}
\textit{(b).}
 Let $\{\bar\e_i\}_{i=1}^\infty$ be an orthonormal basis for $\cH$. We have,
\begin{align*}
\norm{\sI}_{\cL(\cH,\rho)}^2 &= \sum_{i=1}^\infty \norm{\sI \bar\e_i}_\rho^2  \tag{Pythagoras Theorem}\\
&= \sum_{i=1}^\infty \norm{\bar\e_i}_\rho^2  \tag{Definition \ref{def:inclusion}} \\
& = \sum_{i=1}^\infty \int_{\cX} \ip{k(\x,\cdot)}{\bar\e_i}_\cH^2 d\rho(\x) \tag{Reproducing Property}\\
& = \int_{\cX} \sum_{i=1}^\infty  \ip{k(\x,\cdot)}{\bar\e_i}_\cH^2 d\rho(\x) \tag{Fubini's Theorem}\\
& = \int_{\cX} k(x,x) d\rho(\x) \leq \tau^2 < \infty \tag{Assumption \ref{assumption:kernel}}
\end{align*}
For the adjoint $\sI^*$, we have $\norm{\sI^*}_{\cL(\rho,\cH)} = \norm{\sI}_{\cL(\cH,\rho)} < \infty$
\end{proof}

\begin{prop}
\label{lemma:inclusion2} The following properties hold,
\begin{enumerate}
    \item [(a).] The covariance operator and the kernel integral operator satisfy $\C = \sI^* \sI$ and $\L = \sI \sI^*$ respectively.
    \item [(b).] $\C$ and $\L$ are trace-class
\end{enumerate}
\end{prop}
\begin{proof}[Proof of Proposition~\ref{lemma:inclusion2}]
\textit{(a).}
We first show that $\C=\sI^* \sI$. For any $f \in L^2(\cX, \rho)$, we have
\begin{align*}
    \sI^*\sI f &= \sI^* f \tag{Definition~\ref{def:app:inclusion}}\\
    &= \int_\cX k(\x,\cdot) f(\x) d \rho(\x)  \tag{Proposition \ref{lemma:inclusion1}}\\
  & = \C f \tag{Definition~\ref{def:app:C}}
\end{align*}
We now show that $\L=\sI \sI^*$. For any $g \in L^2(\cX, \rho)$, we have
\begin{align*}
    \sI \sI^* g &= \sI \left( \int_\cX k(\x,\cdot) g(\x) d \rho(\x)\right) \tag{Proposition~\ref{lemma:inclusion1}}\\
    &=\int_\cX k(\x,\cdot) g(\x) d \rho(\x) \tag{Definition~\ref{def:app:inclusion}}  \\
    & = \L g
\end{align*}
\textit{(b). } Now we show that $\C$ and $\L$ are trace-class.
\begin{align*}
\norm{\C}_{\cL^1(\cH)} & =\norm{\sI^* \sI}_{\cL^1(\cH)} \tag{Proposition \ref{lemma:inclusion2}}\\
&= \norm{\sI}^2_{\cL^2(\cH)} < \infty \tag{Proposition \ref{lemma:inclusion1}}
\end{align*}
Similarly, 
\begin{align*}
\norm{\L}_{\cL^1(\rho)} & =\norm{\sI \sI^*}_{\cL^1(\rho)} \tag{Proposition \ref{lemma:inclusion2}}\\
&= \norm{\sI}^2_{\cL^2(\cH)} < \infty \tag{Proposition \ref{lemma:inclusion1}}
\end{align*}
\end{proof}

\subsection{Approximation operators \texorpdfstring{$\sA$}{} and  \texorpdfstring{$\fA$}{}}
We first recall the definitions of approximation operators $\sA$ and $\fA$.

\begin{definition}\label{def:app:approximation}[Approximation Operators $\sA$ and $\fA$]
The Approximation operator $\A$ is defined as \\
\[\sA : \cF \rightarrow L^2(\X,\rho),  (\sA \v)(\cdot) = \ip{\z (\cdot)}{\v}, \text{ where } \v \in \cF \]
For an operator $\D \in HS(\cF)$ with rank $k$ with spectral decomposition $\D = \sum_{i =1}^\infty \mu_i \psi_i \otimes \psi_i$, let $\Psi$ be the matrix with eigenvectors $\psi_i$ as columns and $\Phi$ be the matrix with eigenvectors of $\C_m$ as columns (see Definition~\ref{def:C_m}). Define $$\RR^* = \argmin_{\RR^\top\RR = \RR \RR^\top = \I} \norm{\Psi \RR - \Phi}^2_\cF, \ \tilde{\Psi} := \Psi \RR^*$$ 

Let $\tilde{\psi}_i$ be the columns of $\tilde{\Psi}$, define
$$\fA: HS(\cF) \rightarrow HS(\rho), \ \fA \D:= \sum_{i=1}^k \mu_i  \frac{\sA \tilde{\psi}_i}{\sqrt{\langle \C_m\tilde{\psi}_i,\tilde{\psi}_i\rangle_\cF}}  \otimes_{\rho} \frac{\sA \tilde{\psi}_i}{\sqrt{\langle \C_m\tilde{\psi}_i,\tilde{\psi}_i\rangle_\cF}}$$
\end{definition}

Note that the definition of the approximation operator $\fA$ requires knowledge of the co-variance matrix $\C_m$ to find the optimal rotation matrix $\RR^*$, but this is solely for the purpose of analysis and is not used in the algorithm in any form. 

In Proposition \ref{lemma:approximation1}, we show that the adjoint of the Approximation Operator is $$\sA^*: L^2(\X,\rho) \rightarrow \cF, (\sA^* f)_i = \int_\cX f(\x) z_{\omega_i} (\x) d\rho (\x).$$

Moreover, in Proposition \ref{lemma:approximation2} we show that that the approximate covariance operator and the approximate kernel integral operator can be expressed in terms of the Approximation operator $\sA$ as $\C_m = \sA^*\sA$ and $\L_m = \sA \sA^*$.

\begin{prop}\label{lemma:approximation1} The approximation operator satisfies the following properties,
\begin{enumerate}
    \item [(a).] The adjoint of $\sA$ is $(\sA^* f)_i = \frac{1}{\sqrt{m}}\int_\cX f(\x) z_{\omega_i} (\x) d\rho (\x)$
    \item [(b).] $\sA$ and $\sA^*$ are Hilbert-Schmidt.
\end{enumerate}
\end{prop}
\begin{proof}[Proof of Proposition~\ref{lemma:approximation1}]
\textit{(a).} First we show that the  adjoint of $\sA$ is $(\sA^* f)_i = \frac{1}{\sqrt{m}}\int_\cX f(\x) z_{\omega_i} (\x) d\rho (\x)$. For $\v \in \cF, f \in L^2(\cX,\rho)$, we have,  
\begin{align*}
    \ip{\sA \v}{f}_\rho &= \int_\cX (\A\v)(\x) f(\x) d \rho(\x) \\
    & = \int_\cX \ip{\z(\x)}{\v}_{\cF}f(\x) d \rho(\x) \tag{Definition~\ref{def:app:approximation}} \\
    & = \ip{\int_\cX \z(\x) f(\x) d \rho(\x)}{\v}_\cF \tag{Fubini's Theorem}\\
    & = \ip{\sA^* f}{\v}_\cF
\end{align*}
\textit{(b).} Let $\{\e_i\}$ be an orthonormal basis for $\cF$.
\begin{align*}
 \norm{\sA}^2_{\cL(\cF,\rho)} & = \sum_{i=1}^m  \norm{\sA \e_i}^2_\rho \tag{Pythagoras Theorem}\\
  & = \sum_{i=1}^m  \norm{\ip{\z(\cdot)}{\e_i}_\cF}^2_\rho \tag{Definition \ref{def:approximation}}\\
  & = \sum_{i=1}^m \int_\cX (\ip{\z(\x)}{\e_i}_\cF)^2 d \rho(\x) \\
  & \leq \sum_{i=1}^m \int_\cX \norm{\z(\x)}_\cF^2 d \rho(\x) < m \tau^2 < \infty 
\end{align*}
where third last and second inequality follows from Cauchy Schwartz inequality and Assumption \ref{assumption:kernel} respectively. \\
Similarly, to show $\sA^*$ is Hilbert-Schmidt, we note that $\norm{\sA^*}_{\cL^2(\rho,\cF)} = \norm{\sA}^2_{\cL^2(\cF,\rho)} < \infty$
\end{proof}

In the following proposition, we show how the Covariance operator $\C_m$ and kernel Integral operator $\L_m$ are related.
\begin{prop} \label{lemma:approximation2}
The following properties hold,
\begin{enumerate}
    \item [(a).] $\C_m$ and $\L_m$ satisfy that $\C_m = \A^*\A, \L_m = \A\A^*$
    \item [(b).] $\C_m$ and $\L_m$ are trace-class.
 \end{enumerate}
\end{prop}
\begin{proof}[Proof of Proposition \ref{lemma:approximation2}] \textit{(a).} We first show the first part of the Proposition. For any $\v \in \cF$, we have,
\begin{align*} 
\sA^* \sA \v &= \int_\cX \ip{\z(\x)}{\v}_\cF \z (\x) d \rho (\x) \tag{Definition~\ref{def:app:approximation} and Proposition~\ref{lemma:approximation1}}\\
& = \expectation{\rho}{\z(\x) \otimes_\cF \z(\x)} \v \\
& = \C_m \v \tag{Definition~\ref{def:app:C_m}}
\end{align*}
For any $g \in L^2(\cX,\rho)$,
\begin{align*}
    \sA \sA^* g &= \frac{1}{m}\sum_{i=1}^m  z_{\omega_i}(\cdot){\int_\cX z_{\omega_i}(\x) g(\x) d \rho(\x)} \tag{Definition~\ref{def:app:approximation} and Proposition~\ref{lemma:approximation1}}\\
    & = \int_{\cX} \sum_{i=1}^m \frac{1}{\sqrt{m}}z_{\omega_i}(\cdot)\frac{1}{\sqrt{m}}z_{\omega_i}(\x) g(\x) d \rho(\x) \tag{Fubini's Theorem}\\
    & = \int_{\cX} \ip{\z(\x)}{\z(\cdot)}_\cF g(\x) d \rho(\x) \\
    & = \int_{\cX} k_m(\x,\cdot) g(\x) d \rho(\x,t) \tag{Definition of the approximate kernel mapping}\\
    & = \L_m g \tag{Definition~\ref{def:L_m}}
\end{align*}
\textit{(b). } Now we show that $\C_m$ and $\L_m$ are trace-class.
\begin{align*}
\norm{\C_m}_{\cL^1(\cF)} & =\norm{\sA^* \sA}_{\cL^1(\cF)}\tag{Proposition \ref{lemma:approximation2}} \\
&= \norm{\sA}^2_{\cL^2(\cF)} < \infty \tag{Proposition \ref{lemma:approximation1}}
\end{align*}
Similarly, 
\begin{align*}
\norm{\L_m}_{\cL^1(\rho)} & =\norm{\sA \sA^*}_{\cL^1(\rho)} \tag{Proposition \ref{lemma:approximation2}}\\
&= \norm{\sA^*}^2_{\cL^2(\rho)} < \infty \tag{Proposition \ref{lemma:approximation1}}
\end{align*}
\end{proof}

\begin{figure}
    \centering
\begin{tikzpicture}[scale = 0.5]
\draw[blue,fill=blue!5] plot [smooth cycle] coordinates {(1.0,.1)(2.8,.5)(2.8,2.8)(1.4,2.5)} 
node[inner sep=0.8pt, label={[xshift=-0.6cm, yshift=-0.3cm]:{\footnotesize $\cX$}}] (x,t) at (1.8,1.8) {};
\draw plot [smooth cycle] coordinates {(0,-1.2)(0,-1.2)(0,-1.2)(0,-1.2)} 
node[inner sep=0.8pt,label={[xshift=-0.5cm, yshift=-.5cm]:{\footnotesize $\mathfrak{L}$}}](HSX) at (0,-1.2) {};
\draw[red,fill=red!5] plot [smooth cycle] coordinates {(-1.0,7.8)(0.8,8.2)(0.8,10.5)(-0.4,10.2)} 
node[inner sep=0.8pt, label={[xshift=-0.8cm, yshift=0.2cm]:{\scriptsize $HS(\cF)$}}] (HSF) at (0,9.1) {};    
\draw[blue,fill=blue!5]  plot [smooth cycle] coordinates {(1.0,5.1)(2.8,5.5)(2.8,7.8)(1.4,7.5)} 
node[inner sep=0.8pt, label={[xshift=-0.6cm, yshift=-0.7cm]:{\scriptsize $\cF$}}] (F) at (1.8,6.8) {};
\draw[red,fill=red!5] plot [smooth cycle] coordinates {(9.7,-2.25)   (11.3,-2.25) (11.3,0)  (9.5,-0.10)   } 
node[inner sep=0.8pt, label={[xshift=1.2cm, yshift=-0.6cm]:{\scriptsize $HS(\cH)$}}] (HSH) at (10.2,-1.2) {};     
\draw[blue,fill=blue!5] plot [smooth cycle] coordinates {(7,0.25)   (9,0.5) (9,2.65)  (7.8,2.75)   } 
node[inner sep=0.8pt, label={[xshift=0.8cm, yshift=-0.3cm]:{\footnotesize $\cH$}}] (H) at (8,1.8) {};
\draw[red,fill=red!5] plot [smooth cycle] coordinates {(9.2,8.25)   (11,8.5) (11,10.65)  (9.4,10.75)  } 
node[circle, fill=black, inner sep=0.8pt, label={[xshift=1.1cm, yshift=0.2cm]:
{\scriptsize $HS(\rho)$}}] (HSL2) at (10.0,9.2) {};  
\draw[blue,fill=blue!5] plot [smooth cycle] coordinates {(7,5.25)   (9,5.5) (9,7.65)  (7.8,7.75)  } 
node[inner sep=0.8pt, label={[xshift=1.1cm, yshift=-0.7cm]:{\scriptsize $L^2 (\cX,\rho)$}}] (L2) at (8,6.7) {}; 
\draw[-{Straight Barb[length=5pt,width=5pt]}] (F) edge[out=30, in=150] node[above] {{\footnotesize $\sA $}} (L2);
\draw[-{Straight Barb[length=5pt,width=5pt]},dashed] (L2) edge[out=210, in=-30] node[below] {{\footnotesize $\sA^* $}} (F);
\draw[-{Straight Barb[length=5pt,width=5pt]}] (HSF) edge[out=30, in=150] node[above] {{\footnotesize $\fA $}} (HSL2);
\draw[-{Straight Barb[length=5pt,width=5pt]}] (H) edge[out=60, in=-60] node[right] {{\footnotesize $\sI $}} (L2);
\draw[-{Straight Barb[length=5pt,width=5pt]},dashed] (L2) edge[out=-120, in=120] node[left] {{\footnotesize $\sI^* $}} (H);
\draw[-{Straight Barb[length=5pt,width=5pt]}] (HSH) edge[out=60, in=-60] node[right] {{\footnotesize $\fI $}} (HSL2);
\draw[-{Straight Barb[length=5pt,width=5pt]},dashed] (X) edge[out=120, in=-120] node[left] {{\footnotesize $\z $}} (F);
\draw[-{Straight Barb[length=5pt,width=5pt]},dashed] (X) edge[out=-30, in=210] node[above] {{\footnotesize$k(\x,\cdot) $}} (H);
\draw[-{ [length=5pt,width=5pt]},dashed] (HSF) edge[out=-120, in=120] node[above] {} (HSX);
\draw[-{Straight Barb[length=5pt,width=5pt]},dashed] (HSX) edge[out=-30, in=210] (HSH);

\draw[->,  very thick] (F) -- (HSF) ;
\draw[->,  very thick] (H) -- (HSH) ;
\draw[->,  very thick] (L2) -- (HSL2) ;
\end{tikzpicture} 
    \caption{Maps between the data domain ($\cX$), space of square integrable functions on $\cX$ ($L^2(\cX, \rho)$), the RKHS of kernel $k(\cdot, \cdot)$, and RKHS of the approximate feature map, as well as maps between Hilbert-Schmidt operators on these spaces.}
    \label{fig:maps}
\end{figure}
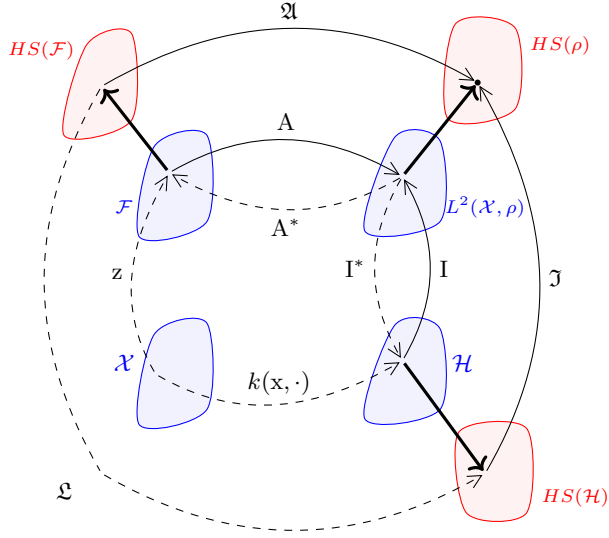

\subsection{Kernel integral operator \texorpdfstring{$\L$}{} and its approximation \texorpdfstring{$\L_m$}{}}
We first recall the definition of Kernel integral operator $\L$ and its approximation $\L_m$.

\begin{definition}
The kernel integral operator $\L: L^2(\cX, \rho) \rightarrow L^2(\cX,\rho)$ is defined as follows:
\[\L g = \int_\cX k(\x,\cdot)g(\x)d \rho(\x) \ \forall \ g \in L^2(\cX,\rho)\]
\end{definition}

\begin{definition}
$\L_m: L^2(\cX,\rho) \rightarrow L^2(\cX,\rho)$ is the (approximated) kernel integral operator, defined as:
$$(\L_m g)(\cdot) = \int_{\cX} k_m(\x,\cdot) g(\x) d \rho (\x)  $$
\end{definition}
We now show in Proposition~\ref{lem:eigsL} that spectral decomposition of the kernel integral operator $\L$ can be given in terms of the eigenfunctions and the eigenvalues of the covariance operator $\C$.
\begin{prop}\label{lem:eigsL}
The spectral decomposition of $\L$ is:
\[\L = \sum_{i=1}^\infty \bar{\lambda}_i \frac{\sI \bar{\phi}_i}{\sqrt{\bar{\lambda}_i}} \otimes_\rho \frac{\sI \bar{\phi}_i}{\sqrt{\bar{\lambda}_i}}\]
where $\frac{\sI \bar{\phi_i}}{\sqrt{\bar{\lambda_i}}}$ are the (unit norm) eigenfunctions of $\L$ with eigenvalues $\bar{\lambda}_i$
\end{prop}
\begin{proof}[Proof of Proposition~\ref{lem:eigsL}]
First we show that the operators $\L$ and  $\sum_{i=1}^\infty \bar{\lambda}_i \frac{\sI \bar{\phi}_i}{\sqrt{\bar{\lambda}_i}} \otimes_\rho \frac{\sI \bar{\phi}_i}{\sqrt{\bar{\lambda}_i}}$ agree on any function $g \in L^2(\cX, \rho)$. 
\begin{align*}
    \sum_{i=1}^\infty \bar{\lambda}_i \frac{\sI \bar{\phi}_i}{\sqrt{\bar{\lambda}_i}} \otimes_\rho \frac{\sI \bar{\phi_i}}{\sqrt{\bar{\lambda}_i}} g &=\sum_{i=1}^\infty \sI \bar{\phi}_i \ip{\sI \bar{\phi_i}}{g}_\rho \tag{Definition of outer product} \\
    &=\sum_{i=1}^\infty \bar{\phi}_i \ip{ \bar{\phi_i}}{g}_\rho \tag{Definition~\ref{def:app:inclusion}}\\
   &=\sum_{i=1}^\infty  \bar{\phi}_i \int_{\cX} \ip{\bar{\phi}_i}{k(\x,\cdot)}_\cH g(\x,t) d \rho(\x,t)\tag{Reproducing property}  \\
   &= \sum_{i=1}^\infty\bar{\phi}_i \ip{\bar{\phi}_i}{\int_{\cX} k(\x,\cdot)g(\x,t) d\rho(\x,t)}_\cH \tag{Fubini's Theorem}\\
   &=\int_{\cX} k(\x,\cdot)g(\x,t) d\rho(\x,t) \tag{Pythagoras theorem} \\
   &= \L g
\end{align*}
where the second to last equality holds by Pythagoras theorem, since $\bar{\phi_i}$'s are basis for $\cH$. Now we show that $\frac{\sI \bar{\phi_i}}{\sqrt{\bar{\lambda_i}}}$ are eigenfunctions of $\L$ with eigenvalues $\bar{\lambda_i}$.

\begin{align*}
    \L \frac{\sI \bar{\phi_i}}{\sqrt{\bar{\lambda_i}}} &= \sI \sI^* \frac{\sI \bar{\phi_i}}{\sqrt{\bar{\lambda_i}}} \tag{Proposition~\ref{lemma:inclusion2}} \\
    & = \sI \C \frac{\bar{\phi_i}}{\sqrt{\bar{\lambda_i}}} \tag{Proposition~\ref{lemma:inclusion2}} \\
    & = \bar{\lambda_i} \frac{\sI  \bar{\phi_i}}{\sqrt{\bar{\lambda_i}}} \tag{$\bar{\phi_i}$ is an eigenfunction of $\C$}\\
\end{align*}
Moreover, they have unit norms and are mutually orthogonal.

\begin{align*}
     \ip{\frac{\sI \bar{\phi_i}}{\sqrt{\bar{\lambda_i}}}}{\frac{\sI \bar{\phi_j}}{\sqrt{\bar{\lambda_j}}}}_\rho &= \frac{1}{\sqrt{\bar{\lambda_i} \bar{\lambda_j}}} \ip{\sI^*\sI \bar{\phi_i}}{\bar{\phi_j}}_\cH \tag{Proposition~\ref{lemma:inclusion1}} \\ 
    & =\frac{1}{\sqrt{\bar{\lambda_i} \bar{\lambda_j}}}\ip{\C \bar{\phi_i}}{\bar{\phi_j}}_\cH \tag{Proposition~\ref{lemma:inclusion2}} \\
    & =\frac{1}{\sqrt{\bar{\lambda_i} \bar{\lambda_j}}} \lambda_i \delta_{ij} = \delta_{ij}
\end{align*}

\end{proof}

Similarly, we characterize the relationship between the spectral decomposition of the approximate kernel integral operator $\L_m$ and the approximate covariance operator $\C_m$ in Proposition~\ref{lem:eigsL_m}.
\begin{prop}\label{lem:eigsL_m}
The spectral decomposition of $\L_m$ is:
\[\L_m = \sum_{i=1}^m \lambda_i \frac{\sA \phi_i}{\sqrt{\lambda_i}} \otimes_\rho \frac{\sA \phi_i}{\sqrt{\lambda_i}}\]
where $\frac{\sA \phi_i}{\sqrt{\lambda_i}}$ are the (unit norm) eigenfunctions of $\L_m$ with eigenvalues $\lambda_i$
\end{prop}
\begin{proof}[Proof of Proposition~\ref{lem:eigsL_m}]
First we prove that $\L_m$ and $\sum_{i=1}^m \lambda_i \frac{\sA \phi_i}{\sqrt{\lambda_i}} \otimes_\rho \frac{\sA \phi_i}{\sqrt{\lambda_i}}$ agree on any function $g$ in $L^2(\cX, \rho)$:
\begin{align*}
    \sum_{i=1}^m \lambda_i \frac{\sA \phi_i}{\sqrt{\lambda_i}} \otimes_\rho \frac{\sA \phi_i}{\sqrt{\lambda_i}} g &= \sum_{i=1}^m \ip{A \phi_i}{g}_\rho A\phi_i \tag{Definition of outer product}\\
    & = \sum_{i=1}^m \int_\cX \ip{\z(\x,t)}{\phi_i}_\cF g(\x,t) d \rho(\x,t)  \ip{\z(\cdot)}{\phi_i}_\cF \tag{Definition~\ref{def:app:approximation}}\\
    & =  \int_\cX \sum_{i=1}^m\ip{\z(\x,t)}{\phi_i}_\cF \ip{\z(\cdot)}{\phi_i}_\cF  g(\x,t) d \rho(\x,t) \tag{Fubini's Theorem} \\
    & = \int_\cX \sum_{i=1}^m \frac{1}{\sqrt{m}}z_{\omega_i}(\x,t)\frac{1}{\sqrt{m}}z_{\omega_i}(\cdot)  g(\x,t) d \rho(\x,t) \tag{Change of basis} \\
    & =\int_\cX \ip{\z(\x,t)}{\z(\cdot)}_\cF  g(\x,t) d \rho(\x,t) \\
    & = \int_\cX k_m(\x,\cdot)  g(\x,t) d \rho(\x,t) \\
    & = \L_m g
\end{align*}
The fourth  equality follows from the fact that inner products are invariant under orthogonal change of basis. Now we show that $\frac{\sA \phi_i}{\sqrt{\lambda_i}}$ are eigenfunctions of $\L_m $ with the corresponding eigenvalue $\lambda_i$
\begin{align*}
    \L_m \frac{\sA \phi_i}{\lambda_i} &= \sA \sA^* \frac{\sA \phi_i}{\sqrt{\lambda_i}} \tag{Proposition~\ref{lemma:approximation2}} \\
    &= \frac{\sA \C_m \phi_i}{\sqrt{\lambda_i}}\tag{Proposition~\ref{lemma:approximation2}} \\
    &= \lambda_i \frac{\sA \phi_i}{\sqrt{\lambda_i}} \tag{$\phi_i$ is an eigenvector of $\C_m$}
\end{align*}
Moreover, they have unit norms and are mutually orthogonal.
\begin{align*}
    \ip{\frac{\sA \phi_i}{\sqrt{\lambda_i}}}{\frac{\sA \phi_j}{\sqrt{\lambda_j}}}_\rho  &= \frac{1}{\sqrt{\lambda_i \lambda_j}}\ip{\sA^* \sA \phi_i}{\phi_j}_\cF \tag{Proposition~\ref{lemma:approximation1}} \\
    & =  \frac{1}{\sqrt{\lambda_i \lambda_j}} \ip{\C_m \phi_i}{\phi_j}_\cF \tag{Proposition~\ref{lemma:approximation2}} \\
    & = \frac{1}{\sqrt{\lambda_i \lambda_j}} \lambda_i \delta_{ij} = \delta_{ij} %\\
\end{align*} 
which completes the proof.
\end{proof}

The following is an important lemma which shows that the kernel integral operator and its approximation can be seen as true and empirical covariance operators in $HS(\rho)$ associated with the random variable $z_\omega$. This allows us to use concentration of measure tools to bound the approximation error in $H(\rho)$.
\begin{lemma} \label{lem:lemForLA}
$\L = \expectation{\omega}{z_\omega \otimes_\rho z_\omega }, \L_m = \frac{1}{m}\sum_{i =1}^m z_{\omega_i} \otimes_\rho z_{\omega_i}  $\end{lemma}
\begin{proof}[Proof of Lemma~\ref{lem:lemForLA}]
For any $f,g \in L^2(\cX,\rho)$ it holds that
\begin{align*}
    \ip{\L f}{g}_\rho &= \ip{ \int_\cX k(\x,\cdot) f(\x,t) d \rho(\x,t)}{g}_\rho \tag{Definition~\ref{def:L}} \\
    & = \ip{\int_\cX \int_\Omega z_\omega(\x,t) z_\omega(\cdot) f(\x,t) d \pi(\omega)  d \rho(\x,t)}{g}_\rho \tag{Theorem~\ref{thm:bochner}}\\ 
    & = \ip{\int_\Omega \int_\cX z_\omega(\x,t) f(\x,t) d \rho(\x,t) z_\omega(\cdot) d\pi(\omega)  }{g}_\rho \tag{Fubini's theorem}\\ 
        & = \ip{\int_\Omega \ip{z_\omega(\cdot)}{f}_\rho  z_\omega(\cdot) d\pi(\omega)  }{g}_\rho \\ 
    & = \ip{\int_\Omega z_\omega(\cdot) \otimes_\rho z_\omega(\cdot)  d\pi(\omega) f }{g}_\rho \tag{Definition of outer product} \\ 
    & = \ip{\expectation{\omega}{z_\omega \otimes_\rho z_\omega } f}{g}_\rho
\end{align*}
Similarly, for any $f,g \in L^2(\cX,\rho)$ we have that
\begin{align*}
    \ip{\L_m f}{g}_\rho &= \ip{ \int_\cX k_m(\x,\cdot) f(\x,t) d \rho(\x,t)}{g}_\rho \tag{Definition~\ref{def:L_m}} \\
    & = \ip{\int_\cX \ip{\z(\x,t)}{\z(\cdot)}_\cF f(\x,t)   d \rho(\x,t)}{g}_\rho \tag{Definition of $k_m$} \\ 
     & = \ip{\int_\cX \frac{1}{m}\sum_{i=1}^m z_{\omega_i}(\x,t)  z_{\omega_i}(\cdot)  f(\x,t)   d \rho(\x,t)}{g}_\rho \\ 
       & = \ip{ \frac{1}{m}\sum_{i=1}^m \int_\cX z_{\omega_i}(\x,t) f(\x,t)   d \rho(\x,t) z_{\omega_i}(\cdot)}{g}_\rho  \\
      & = \ip{\frac{1}{m}\sum_{i=1}^m \ip{z_{\omega_i}(\cdot)}{f}_\rho z_{\omega_i}(\cdot)}{g}_\rho \\
       & = \ip{\frac{1}{m}\sum_{i=1}^m z_{\omega_i}(\cdot) \otimes_\rho z_{\omega_i}(\cdot)f }{g}_\rho \tag{Definition of outer product} \\
       & = \ip{\L_m f}{g}_\rho \tag{Definition~\ref{def:L_m}}
\end{align*}
which completes the proof.

\end{proof}

The following Proposition shows the relation between the outer products in two separable Hilbert spaces; this is useful in the proof of the main theorem.
\begin{prop}\label{lem:outerproduct}
For any Hilbert-Schmidt Operator $\B:\cH_1 \rightarrow \cH_2$, it holds that $\B u \otimes_{\cH_2} \B v = \B (u \otimes_{\cH_1} v) \B^*$, where $\u, \v \in \cH_1$.
\end{prop}
\begin{proof}[Proof of Proposition~\ref{lem:outerproduct}]
For any $f \in \cH_2$, the following equalities hold:
\begin{align*}
    (\B u \otimes_{\cH_2} \B v)f &=   \B u \ip{\B v}{f}_{\cH_2} \tag{Definition of outer product}\\
    & =\B u \ip{v}{\B^*f}_{\cH_1} \tag{Proposition~\ref{lemma:approximation1}} \\
    & = \B (u \ip{v}{\B^*f}_{\cH_1}) \\
    & = \B (u \otimes_{\cH_1} v)\B^*f \tag{Definition of outer product}
\end{align*}
\end{proof}

Finally, we state the assumptions that we make throughout the paper.
\begin{assumption}  \label{assumption:app:kernel}
The kernel function $k$ is a Mercer's kernel(see Theorem~\ref{thm:mercer}) and has the following integral representation,
$k(\x,\y) = \int_\Omega z (\x,\omega) z (\y,\omega) d\pi(\omega) \ \forall \x,\y \in \cX$ where 
$(\cH, k)$ is a separable RKHS of real-valued functions on
$\cX$ with a bounded positive definite kernel $k$. We also assume that there exists $\tau >  1$ such that $|z(\x,\omega)|\leq \tau$ for all $x \in \cX, \omega \in \Omega$. Furthermore, we assume that the operator $\L^{\frac12}$ exists.
\end{assumption}

\newpage

\section{Equivalence of optimization problems in \texorpdfstring{$\cH$}{} and \texorpdfstring{$L^2(\cX,\rho)$}{}} 
\label{sec:app:eq}
We now show that Kernel PCA in the RKHS $\cH$ and $L^2(\cX, \rho)$ are equivalent under some assumptions which we show are naturally satisfied in the case of Kernel PCA with random features. Ths allows us to transfer our generalization bounds established in $L^2(\cX,\rho)$ to $\cH$.

The Kernel PCA problem essentially reduces to solving the following optimization problem:
\begin{align*}
\label{opt:opt1}
    \max_{\P \in \cP^k_\cH}{\ip{\P}{\C}_{HS(\cH)}}
\tag{OPT-1}
\end{align*}
For any $\P \in \cP^k_{HS(\cH)}$, by spectral decomposition, $\P$ has an eigendecompostion given by $\P = \sum_{i=1}^k \u_i \otimes_\rho \u_i$ where $\u_i \in \cH, i \in [k]$ are a set of orthonormal functions. We define an operator $\U: \R^k \rightarrow \cH$ such that $\U \b = \sum_{i=1}^k \b_i \u_i$, where $\b \in \R^k$. 

\begin{prop}
\label{prop:eq1}
$\U$ satisfies the following properties.
\begin{enumerate} 
    \item [(a)]$\U$ is Hilbert-Schmidt.
    \item [(b)]The adjoint of $\U$ is  $\U^*: \cH \rightarrow \R^k$ such that $(\U^*f)_i = \ip{\u_i}{f}_\cH$ where $f \in \cH$.
    \item [(c)] $\P = \U \U^*$ and $ \U^* \U= \I_k$
\end{enumerate}
\end{prop}
\begin{proof}
\textit{(a)} First we show that the operator $\U$ is Hilbert-Schmidt. Let $\{\e_i\}_{i=1}^k$ be the canonical basis of $\R^k$.
\begin{align*}
    \norm{\U}^2_{\cL^2(\R^k,\cH)} &= \sum_{i=1}^k\norm{\U \e_i}^2_{\cH} \tag{Pythagoras Theorem}\\
    &= \sum_{i=1}^k \norm{\u_i}^2_{\cH} = k 
\end{align*}
\textit{(b)} Let $\U^*$ be the adjoint of $\U$. We now show that $(\U^*f)_i = \ip{\u_i}{f}_\cH$. For any $\b \in \R^k, f \in \cH$,
 \begin{align*}
     \ip{\U^*f}{\b} &= \ip{f}{\U \b}_\cH \\
     &=\ip{f}{\sum_{i=1}^k\b_i \u_i}_\cH \\
     & = \sum_{i=1}^k \ip{f}{\u_i}_\cH \b_i  \\
     & = \ip{\d}{\b}
 \end{align*}
 where $\d \in \R^k, \d_i = \ip{f}{\u_i}_\cH$

\textit{(c)} For the first part, for any $f \in \cH$, we have,
\begin{align*}
    \P f &= \sum_{i=1}^k (\u_i \otimes_\cH \u_i)f \\
    &= \sum_{i=1}^k \ip{\u_i}{f}_\cH \u_i \\
    & = \sum_{i=1}^k (\U^* f)_i\u_i \\
    & = \U \U^* f
\end{align*}

Now we show that the constraint $\P \in \cP_{\cH}^k$ reduces to $\U^* \U= \I_k$.

For any $\b \in \R^k$,
\begin{align*}
    \U^* \U  \b &= \sum_{i=1}^k  \b_i \U^*\u_i \\
    &= \sum_{i=1}^k \b_i \d_i  \\
\end{align*}
where $\d_i \in \R^k, (d_i)_j = \ip{\u_j}{\u_i}_\cH$. Note that since $u_i$'s are orthonormal functions, therefore, $\d_i = \e_i$, where $\e_i$'s is the canonical basis of $\R^k$. Therefore,
\begin{align*}
     \U^* \U  \b = \sum_{i=1}^k \b_i \e_i = \b
\end{align*}
\end{proof}
We now write the optimization problem \ref{opt:opt1} in terms of $\U$ as,
\begin{align*}
\label{opt:opt2}
    \max_{\U^* \U = \I_k}{\ip{\U \U^*}{\C}_{HS(\cH)}}
\tag{OPT-2}
\end{align*}

Now consider $\v_i \in L^2(\cX, \rho), i \in [k]$ such that $\u_i = \sI^*\v_i$. Note that the existence of $\v_i$ is guaranteed from the construction of RKHS from eigenfunctions of $\L$ (for details, see \cite[Theorem 51]{sejdinovic2012rkhs}). We now define an operator $\V: \R^k \rightarrow L^2(\cX, \rho)$ such that $\V \b = \sum_{i=1}^k \b_i \v_i$, where $\b \in \R^k$. We have the following proposition about $\V$.

\begin{prop} 
\label{prop:eq2}
$\V$ satisfies the following properties,
\begin{enumerate}
    \item [(a)] $\V$ is Hilbert-Schmidt.
    \item [(b)] The adjoint of $\V$ is $\V^*: L^2(\cX, \rho) \rightarrow \cH$, defined as $(V^*f)_i = \ip{v_i}{f}_\rho$ 
    \item [(c)] $\ip{ \V \V^*}{\L^2}_{HS(\rho)} = \ip{\U\U^*}{\C}_{HS(\cH)}$ and $\V^* \L \V = \U^*\U = \I_k$
\end{enumerate}
\end{prop}
\begin{proof}
The proofs of \textit{(a)} and \textit{(b)} are similar to that of Proposition \ref{prop:eq1}. \\
\textit{(c)}. We start with the first part. The objective in terms of $\V$ is,
\begin{align*}
    \ip{\P}{\C}_{HS(\cH)} &= \ip{\sum_{i=1}^k \u_i \otimes_\cH \u_i}{\C}_{HS(\cH)} \\
    & = \ip{\sum_{i=1}^k \sI^* \v_i \otimes_\cH \sI^* \v_i}{\C}_{HS(\cH)} \tag{$\u_i = \sI^* \v_i$}\\
    & = \ip{\sum_{i=1}^k \sI ( \v_i \otimes_\rho  \v_i)\sI^*}{\C}_{HS(\cH)} \tag{Proposition \ref{lem:outerproduct}}\\
    & = \ip{\sum_{i=1}^k \v_i \otimes_\rho  \v_i}{\sI\C\sI^*}_{HS(\rho)} \tag{Definition of adjoint}\\
    & = \ip{\sum_{i=1}^k \v_i \otimes_\rho  \v_i}{\sI\I^*\sI\sI^*}_{HS(\rho)} \tag{Proposition \ref{lemma:inclusion2}}\\
    & = \ip{\sum_{i=1}^k \v_i \otimes_\rho  \v_i}{\L^2}_{HS(\rho)} \tag{Proposition \ref{lemma:inclusion2}}\\
     & = \ip{ \V \V^*}{\L^2}_{HS(\rho)} \\
\end{align*}
For the second part, for any $\b \in \R^k$, we have,
\begin{align*}
    \U^* \U  \b &= \sum_{i=1}^k  \b_i \U^*\u_i \\
    & = \sum_{i=1}^k \b_i \U^* \sI^* \v_i \tag{$\u_i = \sI^* \v_i$}\\
    &= \sum_{i=1}^k \b_i \d_i 
\end{align*}
where $\d_i \in \R^k, (\d_i)_j = \ip{\u_j}{\sI^* \v_i}_\cH = \ip{\sI^* \v_j}{\sI^* \v_i}_\cH = \ip{\v_j}{\sI \sI^* \v_i}_\rho = \ip{\v_j}{\L \v_i}_\rho$, where the third equality follows from the property of adjoints, and the last equality because $\L = \sI \sI^*$.  \\
Since $\U^* \U \b = \b$, so $\d_i = \e_i$. Therefore, we get $\ip{\v_j}{\L \v_i} = \delta_{ij}$ \\

Let us now look at the $j^\text{th}$ element of $\V^* \L \V\b$,
\begin{align*}
    (\V^* \L \V\b)_j & = (\V^* \L \sum_{i=1}^k \b_i \v_i)_j \tag{Definition of $\V$} \\
    &= \sum_{i=1}^k (\b_i \V^* \L \v_i)_j  \\
    & = \sum_{i=1}^k \b_i \ip{\v_j}{\L\v_i}_\rho \\
    & = \sum_{i=1}^k \b_i \delta_{ij} = \b_j
\end{align*}
\end{proof}
We can now restate the optimization problem in terms of $\V$.
\begin{align*}
\label{opt:opt3}
    \max_{\V^* \L \V = \I_k}\ip{\V\V^*}{\L^2}_{HS(\rho)} \tag{OPT-3}
\end{align*}

Now, let $\w_i = \L^{1/2}\v_i$. Note that $\w_i$ is well-defined since we assume that $\L^{1/2}$ exists (See Assumption \ref{assumption:kernel}). Define $\W: \R^k \rightarrow L^2(\cX, \rho)$, such that $\W \b = \sum_{i=1}^k \b_i \w_i$.
\begin{prop}
\label{prop:eq3}
$\W$ satisfies the following properties,
\begin{enumerate}
    \item [(a)] $\W$ is Hilbert-Schmidt.
    \item [(b)] The adjoint of $\W$ is $\W^*:L^2(\cX,\rho) \rightarrow \R^k$ $(\W^*f)_i = \ip{\w_i}{f}_\rho$.
    \item [(c)] $\W = \L^{1/2}\V$, $\ip{\V\V^*}{\L^2}_{HS(\rho)} = \ip{\W\W^*}{\L}_{HS(\rho)}$ and $\W^*\W = \V^* \L \V = \I_k$
\end{enumerate}
\end{prop}
\begin{proof}
The proofs of \textit{(a)} and \textit{(b)} are similar to that of Proposition \ref{prop:eq1}. \\
\textit{(c)} For the first part, for any $\b \in \R^k$, we have
\begin{align*}
    \W \b &= \sum_{i=1}^k \b_i \w_i \\
    & = \sum_{i=1}^k \b_i \L^{1/2} \v_i \\
    & = \L^{1/2}\sum_{i=1}^k \b_i \v_i \\
    & = \L^{1/2} \V \b
\end{align*}
Note that  since $\L$ is self-ajoint, $\L^{1/2}$ is self-adjoint too. The objective in terms of $\W$ is 
\begin{align*}
    \ip{\V\V^*}{\L^2}_{HS(\rho)} &=  \ip{\L^{1/2}\V\V^*\L^{1/2}}{\L}_{HS(\rho)} \tag{Definition of adjoint}\\
    & = \ip{\W\W^*}{\L}_{HS(\rho)}
\end{align*}
Equivalently, we can restate the constraint in terms of $\W$ as,
\begin{align*}
    \V^* \L \V &= \V^* \L^{1/2} \L^{1/2} \V \\
    & = (\L^{1/2}\V)^*(\L^{1/2}\V) \\
    & = \W^*\W = \I_k
\end{align*}
\end{proof}
We now restate the optimization problem in terms of $\W$.
\begin{align*}
\label{opt:opt4}
    \max_{\W^*\W = \I_k}\ip{\W \W^*}{\L}_{HS(\rho)}
\tag{OPT-4}
\end{align*}

We now state this equivalence of objective in the following Lemma. 
\begin{lemma}[Equivalence of Objective]
\label{prop:eq4}
$$\ip{\P}{\C}_{HS(\cH)} = \ip{\W\W^*}{\L}_{HS(\rho)}$$ where the relation between $\P$ and $\W$ is presented via Propositions \ref{prop:eq1},\ref{prop:eq2} and \ref{prop:eq3}.
\end{lemma}
\begin{proof}
One direction of implication simply simply follows from the construction in Propositions \ref{prop:eq1}, \ref{prop:eq2} and \ref{prop:eq3}. In particular, from Propositions \ref{prop:eq1}, \ref{prop:eq2} and \ref{prop:eq3}, we conclude that \ref{opt:opt1} $\implies$  \ref{opt:opt2} $\implies$  \ref{opt:opt3}  $\implies$ \ref{opt:opt4}. It is easy to see that \ref{opt:opt3} $\implies$ \ref{opt:opt2} $\implies$  \ref{opt:opt1} where the first implication simply follows from the construction of $\u_i$'s and the second from  Proposition \ref{prop:eq1}. However, showing that \ref{opt:opt4} $\implies$ \ref{opt:opt3} is conditioned on $\w_i$'s lying in the range of $\L^{1/2}$ because otherwise there might not exist $\v_i$'s such that $\w_i = \L^{1/2} \v_i$. In Lemma \ref{lem:eq5}, we show that when using random features, with the approximation operator defined in Definition \ref{def:approximation}, the functions obtained via random feature approximation lies in the range of $\L^{1/2}$ with probability 1 on the support of $\pi$. This establishes the equivalence claimed.
\end{proof}

We now formally show that vectors from $\cF$ lifted to $L^2(\cX,\rho)$ via the approximation operator $\sA$ lie in the range of $\L^{1/2}$ almost surely with respect to measure $\pi$. The proof of the following lemma closely follows  \citep[Lemma 2]{rudi2017generalization}.
\begin{lemma}
\label{lem:eq5}
For every $\v \in \cF$, $\sA \v \in L^2(\cX,\rho)$ lies in the range of $\L^{1/2}$ almost surely on the support of $\pi$.
\end{lemma}
\begin{proof}

Let $\Pi \in HS(\rho)$ denote the projection operator projecting to the range of $\L^{1/2}$. Then $(\I_\rho-\Pi)\L^{1/2} f = \0 \ \forall \ f \in L^2(\cX, \rho)$ as $(\I_\rho - \Pi)$ is the projection to the orthogonal complement to the range of $\L^{1/2}$. From this, we have, $\trace{(\I_\rho-\Pi)\L^{1/2}\L^{1/2}(\I_\rho -\Pi)} = \trace{(\I_\rho-\Pi)\L(\I_\rho -\Pi)}= 0$.
\begin{align*}
    \trace{(\I_\rho - \Pi)\L(\I_\rho-\Pi)} &= \trace{(\I_\rho - \Pi)\int_{\Omega} z_\omega \otimes_\rho z_\omega d\pi(\omega)(\I_\rho - \Pi)} \\
    &=\int_{\Omega} \trace{(\I_\rho - \Pi) ( z_\omega \otimes_\rho z_\omega )(\I_\rho - \Pi)} d\pi(\omega) \\
    &=\int_{\Omega} \trace{(\I_\rho - \Pi) z_\omega \otimes_\rho (\I_\rho - \Pi) z_\omega} d\pi(\omega) \\
    &=\int_{\Omega} \norm{(\I_\rho - \Pi) z_\omega}_\rho^2 d\pi(\omega) = 0
\end{align*}
From the above equation, we see that $\norm{(\I_\rho - \Pi)z_\omega}_\rho = 0$ almost surely on the support of $\pi$. This implies that $(\I_\rho - \Pi)z_\omega = \0$ a.s. on the support of $\pi$. 

Now we show that all functions of interest i.e. anything lifted from $\cF$ to $L^2(\cX, \rho)$ lie in the range of $\L^{1/2}$. Let $\v \in \cF, f \in L^2(\cX, \rho)$.
\begin{align*}
    \ip{(\I_\rho-\Pi)f}{\sA \v}_\rho &= \ip{\sA^*(\I_\rho-\Pi)f}{\v}_\cF \\
    &= \sum_{i=1}^m ((\I_\rho-\Pi) z_{\omega_i}f) \v_i = 0
\end{align*}

This is because $\omega_i$'s are drawn from $\pi$, and we already argued that  $(\I_\rho - \Pi)z_\omega = \0$ a.s. on the support of $\pi$. Since this holds for any $\v \in \cF$ and $f \in L^2(\cX, \rho)$, this implies that $\sA \v$ lies in the range of $\L^{1/2}$ for all $\v \in \cF$.
\end{proof}
Moreover, note that since $\L_m = \fA \C_m$ (See Proposition \ref{lem:eigsL_m}), the eigenfunctions of $\L_m$ are the lifted eigenvectors of $\C_m$ from $\cF$ to $L^2(\cX, \rho)$. This equivalence entails that any candidate solution of \ref{opt:opt2} has an equivalent candidate solution for \ref{opt:opt4}, and they would both have the same objective.

As already hinted, the solution of Kernel PCA with random features might not lie in the constraint set of rank $k$ projection operators over $L^2(\cX, \rho)$.  By the equivalence of the optimization problems \ref{opt:opt1} and \ref{opt:opt4},  we violate the constraint in $\cH$ as well. We, however, remarked that we counter this problem by showing a fast $O(1/\sqrt{n})$ speed of convergence to the constraint set in $L^2(\cX, \rho)$. A natural question to ask is does the speed of convergence to the constraint set preserved too? We answer affirmatively as shown below.

We now use this equivalence to give a reduction from a candidate solution \ref{opt:opt4} to \ref{opt:opt2}. Let $\tilde \P = \sum_{i=1}^k \tilde \p_i \otimes_\rho \tilde \p_i$ be the output of some algorithm for Kernel PCA with random features, lifted through the approximation operator $\fA$. We have show in Theorem~\ref{lem:LisProjA} that $\fA \P_{\C_m}^k = \P_{\L_m}^k$ is a rank k projection over $L^2(\cX, \rho)$. Let $\P_{\L_m} = \sum_{i=1}^k \tilde \q_i \otimes \tilde \q_i$.  Since $\tilde \p_i$ and $\tilde \q_i$ lie in the range of $\L^{1/2} , \ \forall \ i \in [k]$ (See Lemma \ref{lem:eq5}), there exists $\p_i$'s and $\q_i$'s such that $\tilde \p_i = \L^{1/2}\p_i$ and $\tilde \q_i = \L^{1/2}\q_i$, $i \in [k]$. Define $ \P := \sum_{i=1}^k \sI^* \p_i \otimes_\cH \sI^*\p_i$, and  $\Q := \sum_{i=1}^k \sI^* \q_i \otimes_\cH \sI^*\q_i$. 

First we show that $\Q$ is a projection operator in $HS(\cH)$.
\begin{align*}
    \ip{\sI^* \q_i}{\sI^* \q_j}_\cH &= \ip{\q_i}{\sI\sI^* \q_j}_\rho \tag{Definition of adjoints}\\
    &=\ip{\q_i}{\L \q_j}_\rho \tag{Proposition~\ref{lemma:inclusion2}} \\
%    &=\ip{\q_i}{\L^{1/2}\L^{1/2}\q_j}\rho \\
    &=\ip{\L^{1/2}\q_i}{\L^{1/2}\q_j}_\rho \tag{Definition of adjoints}\\
    &=\ip{\tilde\q_i}{\tilde \q_j}_\rho = \delta_{ij}
\end{align*}

Now, let us look at the rate of convergence $\P$ to $\cP^k_{HS(\cH)}$.
\begin{lemma}[Equivalence of convergence to the constraint set]
\label{lem:eq6}
$$ d(\bar\P, \cP_{HS(\cH)}^k) \leq \norm{\tilde \P - \fA \C_m}_{HS(\rho)}$$
\end{lemma}
\begin{proof}
\begin{align*}
    d(\bar\P, \cP_{HS(\cH)}^k) & \leq \norm{\P-\Q}_{HS(\cH)} \\
    &= \norm{\sum_{i=1}^k \sI^* \p_i \otimes_\cH \sI^*\p_i - \sI^*\q_i \otimes_\cH\sI^* \q_i}_{HS(\cH)}\\
    &= \norm{\sI (\sum_{i=1}^k  \p_i \otimes_\rho \p_i - \q_i \otimes_\rho \q_i) \sI^*}_{HS(\cH)} \tag{Proposition \ref{lem:outerproduct}}\\
    & =  \norm{\L (\sum_{i=1}^k  \p_i \otimes_\rho \p_i - \q_i \otimes_\rho \q_i)}_{HS(\rho)} \\
    & =  \norm{\L^{1/2} (\sum_{i=1}^k  \p_i \otimes_\rho \p_i - \q_i \otimes_\rho \q_i) \L^{1/2}}_{HS(\rho)} \tag{Cyclic property}\\
    & =  \norm{ \sum_{i=1}^k  \L^{1/2}\p_i \otimes_\rho \L^{1/2}\p_i - \L^{1/2}\q_i \otimes_\rho \L^{1/2}\q_i}_{HS(\rho)} \tag{Proposition \ref{lem:outerproduct}}\\
     & =  \norm{\sum_{i=1}^k  \tilde\p_i \otimes_\rho \tilde\p_i - \tilde\q_i \otimes_\rho \tilde\q_i}_{HS(\rho)} \\
      & =  \norm{\tilde \P - \fA \C_m}_{HS(\rho)}
\end{align*}
\end{proof}

In Lemma \ref{lemma:est1}, we will bound $\norm{\tilde \P - \fA \C_m}_{HS(\rho)}$ which implies the bound given in the main theorem \ref{thm:main_th}.

We now combine the above relations into a definition to lift operators from $HS(\cF)$ to $HS(\cH)$ and then discuss that the operator is well-defined.

\begin{definition}[Operator $\fL$]
\label{def:app:lift}
Let $\tilde \P \in HS(\cF)$ and $\fA \tilde \P = \sum_{i=1}^k \tilde \p_i \otimes_\rho \tilde \p_i$ be $\tilde \P$ lifted to $L^2(\cX,\rho)$. Consider the equivalence relation  $\p_i \sim \p_j$ if $\L^{1/2} \p_i = \L^{1/2}  \p_j$. Let $[\p_i]$ be the equivalence class such that $\L^{1/2} \p_i = \tilde \p_i$.  The operator $\fL: HS(\cF) \rightarrow HS(\cH)$ is defined as, 
$$\fL \hat P = \sum_{i=1}^k \sI^* \p_i \otimes_\cH \sI^* \p_i$$
Here $\sI^*$ is the restriction of the operator $\sI^*$ to the quotient space $L^2(\cX,\rho)/\sim$.
\end{definition}

 We now discuss that the operator $\fL$ is indeed a well defined operator. We guarantee by Lemma \ref{lem:eq5} that there is at least one element in $[\p_i]$ such that $\tilde \p_i = \L^{1/2} \p_i$. It remains to argue that all the elements in the equivalence class $[\p_i]$ are being mapped to the same element in $\cH$ through $\sI^*$. Let $\p_i$ and $\p_j$ be two elements of $[\p_i]$. Since $\L^{1/2} \p_i = \L^{1/2}  \p_j$, therefore $\L^{1/2}(\p_i -  \p_j) = 0$. This implies that $ \p_j = \p_i + \text{Ker}(\L^{1/2})$. Note that any $\r_i \in \text{Ker}(\L^{1/2})$ will be mapped by $\sI^*$ to $\0$, i.e. $\sI^* \r_i = \0$. It follows from linearity of $\sI^*$ that $\sI^*\tilde \p_j = \sI^*\p_i$. Thus this maps an equivalence class to a single element in $\cH$.

\newpage
\section{Proof of the main Theorem}
From we have already established the problems in $HS(\rho)$ and $HS(\cH)$, we focus on error decomposition and bounding the corresponding error terms in $L^2(\cX,\rho)$. Solving KPCA by using a kernel approximation, one needs to consider two different sources of error. First, the error coming from approximating the true kernel operator by random features. Second, the statistical error due to estimating the covariance operator using iid samples from the unknown distribution. Thus, we distinguish between and base our proof around these two sources of error, namely \textit{approximation error} and \textit{estimation error}. In particular we decompose our objective as:
\begin{align*}
\ip{\fI \P^k_\C}{\fI \C}_{HS(\rho)} - \ip{\fA \hat{\P}}{\fI \C}_{HS(\rho)} &= \underbrace{\langle \fI \P^k_\C , \fI \C \rangle_{HS(\rho)} - \langle \fA \P^k_{\C_m}, \fI \C \rangle_{HS(\rho)}}_{\epsilon_a: \ \text{Approximation Error}}
\\ &+ \underbrace{\langle \fA \P^k_{\C_m}, \fI \C \rangle_{HS(\rho)} -  \langle \fA \hat{\P}, \fI \C \rangle_{HS(\rho)}}_{\epsilon_e: \  \text{Estimation Error}}.
\end{align*}
The first term in the decomposition is interpreted as approximation error because it essentially captures the error incurred by approximating the kernel function with random features. The second term in the decomposition is interpreted as estimation error as it is the error incurred in the original statistical estimation problem. In what follows, we give a bound on each of the error terms and provide a detailed analysis. Throughout this section, we use the following Lemma that shows the relation between different projection operators.
\begin{lemma}\label{lem:LisProjA}
$\fI\P^k_\C$ and $\fA \P^k_\L$ are rank $k$ projection operators in $L^2(\cX, \rho)$. Furthermore, it holds that $\fI\P^k_\C = \P^k_\L$ and $\fA\P^k_{\C_m}  = \P^k_{\L_m}$.
\end{lemma}

\begin{proof}[Proof of Lemma~\ref{lem:LisProjA}]
We have 
\begin{align*}
\fI\P^k_\C &= \sum_{i=1}^k \frac{\sI \bar{\phi_i}}{\sqrt{\bar{\lambda_i}}} \otimes_\rho \frac{\sI \bar{\phi_i}}{\sqrt{\bar{\lambda_i}}} = \P^k_\L
\end{align*}
where the second inequality follows from Lemma \ref{lemma:inclusion2}. Similarly,
\begin{align*}
\fA\P^k_{\C_m} &= \sum_{i=1}^k \frac{\sA \phi_i}{\sqrt{\lambda_i}} \otimes_\rho \frac{\sA \phi_i}{\sqrt{\lambda_i}} = \P^k_{\L_m}
\end{align*}
where the second inequality follows from Lemma \ref{lemma:approximation2}.
and $\fA \P^k_\C$
\end{proof}

\subsection{Approximation Error}
The main idea behind controlling the approximation error is to use the local Rademacher complexity of the kernel class~\cite{massart2000some,bartlett2002localized}. More precisely, we use the following result in~\citep{blanchard2007statistical}, which allows us to get rates depending both on the number of features used and the decay of the spectrum of the operator $\C_2$.

\begin{theorem}[\cite{blanchard2007statistical}] \label{thm:blanchard}
Assume $\|\zeta\|^2 \leq M$ almost surely, and let $(\lambda_i)$ denote the ordered eigenvalues of $\C:=\E[\zeta\zeta^\top]$, and further assume that $(\lambda_i)$ are distinct. Let $B_k:=\frac{\sqrt{\E[\langle \zeta,\zeta'\rangle^4]}}{\lambda_k-\lambda_{k+1}}$, where $\zeta'$ is and iid copy of $\zeta$. Then for all $\delta$, with overwhelming probability of at least $1-e^{-\delta}$ it holds that $$\langle P^k_{\hat{\C}^\perp},\C \rangle - \langle P^k_{\C^\perp},\C \rangle \leq 24\kappa(B_k,k,n) + \frac{11\delta(M+B_k)}{n}$$
where $\kappa$ is defined as follows:
$$\kappa(B_k,k,n)=\inf_{h\geq 0 }\left\{ \frac{B_k h}{n} + \sqrt{\frac{k}{n}\sum_{j>h}\lambda_i(\C')} \right\}$$
\end{theorem}

\begin{lemma}[Approximation Error]
\label{lemma: approxError_Appendix}
With probability at least $1 - \frac{\delta}{2}$, we have
\begin{align*} 
   \langle  \P^k_{\L_m^\perp},  \L \rangle -\langle  \P^k_{\L^\perp} ,  \L \rangle \leq 24 \kappa (B_k, k,m) + \frac{11\log{\delta/2}  \tau^2 + 7B_k}{m}
\end{align*}
\end{lemma}
\begin{proof}
We first note that $\fA P_{\C_m}^k = \sum_{i=1}^k\frac{\sA \phi_i}{\sqrt{\lambda_i}} \otimes \frac{\sA \phi_i}{\sqrt{\lambda_i}}$ (see definition of the approximation operator in~\ref{def:app:approximation}). The following holds for the approximation error:
\begin{align}\label{eq:app_err}
\epsilon_a &= \ip{\fI P^k_\C , \fI \C \rangle_{HS(\rho)} - \langle \fA P^k_{\C_m}}{\fI \C}_{HS(\rho)} \\
&= \ip{\sum_{i=1}^k \frac{\sI \bar{\phi_i}}{\sqrt{\bar{\lambda_i}}} \otimes_\rho \frac{\sI \bar{\phi_i}}{\sqrt{\bar{\lambda_i}}}}{\sum_{i \in I \subset \R} \bar{\lambda_i} \left(\frac{\sI \bar{\phi_i}}{\sqrt{\bar{\lambda_i}}} \otimes_\rho \frac{\sI \bar{\phi_i}}{\sqrt{\bar{\lambda_i}}}\right)}_{HS(\rho)} \tag{definition~\ref{def:app:inclusion}}\nonumber\\
&- \ip{\sum_{i=1}^k \frac{\sA \phi_i}{\sqrt{\lambda_i}} \otimes_\rho \frac{\sA \phi_i}{\sqrt{\lambda_i}}}{\sum_{i \in I \subset \R}\bar{\lambda_i} \left(\frac{\sI \bar{\phi_i}}{\sqrt{\bar{\lambda_i}}} \otimes_\rho \frac{\sI \bar{\phi_i}}{\sqrt{\bar{\lambda_i}}}\right)}_{HS(\rho)} \tag{definition~\ref{def:app:inclusion},~\ref{def:app:approximation}}\nonumber\\
& = \langle  P^k_\L ,  \L \rangle_{HS(\rho)} - \langle  P^k_{\L_m},  \L \rangle_{HS(\rho)} \tag{Lemma~\ref{lem:eigsL} and Lemma~\ref{lem:LisProjA}}\nonumber\\
&= \langle  P^k_{\L_m^\perp},  \L \rangle_{HS(\rho)} -\langle  P^k_{\L^\perp} ,  \L \rangle_{HS(\rho)} \tag{properties of the orthogonal subspace} \nonumber
\end{align}

We have already showed that $\L$ and $\L_m$ in the right hand side of the equation~\eqref{eq:app_err} are true and empirical covariance operators respectively (see Lemma~\ref{lem:lemForLA}). As required by  Theorem \ref{thm:blanchard}, we need to show that norm of the random variables $z_\omega$ are bounded. We have 
\begin{align*}
    \norm{z_\omega}^2 &= \ip{z_\omega}{z_\omega}_\rho \\
     &= \int_\cX z_\omega(\x,t)^2 d \rho(\x,t) \\
     &\leq \tau^2
\end{align*}
where the last inequality follows Assumption \ref{assumption:kernel}.

Invoking Theorem \ref{thm:blanchard}, we have with probability at least $1 - \delta$,

\begin{align*}
   \langle  P^k_{\L_m^\perp},  \L \rangle -\langle  P^k_{\L^\perp} ,  \L \rangle \leq 24 \kappa (B_k, k,m) + \frac{ 11 \log \delta \tau^2 + 7B_k}{m}
\end{align*}
where $\kappa (B_k, k,m) = \inf_{h\geq 0} \left\{ \frac{B_k h}{m} + \sqrt{\frac{k \sum_{j>h} \lambda_j(C_2')}{m}}\right\}$. 
\end{proof}

\begin{lemma}[Approximation Error - Good decay]
\label{lemma: approxErrorGoodDecay}
When the spectrum of operator $\C_2'$ has an exponential decay, i.e. $\lambda_j(\C'_2) = \alpha^j$ for some $\alpha < 1$, then with probability at least $1 - \delta$, we have
\begin{align*} 
   \langle  \P^k_{\L_m^\perp},  \L \rangle -\langle  \P^k_{\L^\perp} ,  \L \rangle \leq \frac{24 B_k \log{m}}{\log{1/\alpha}m} + \frac{k + (1-\alpha)(11\log \delta \tau^2 + 7B_k)}{(1-\alpha)m}
\end{align*}
\end{lemma}
\begin{proof}
When $\lambda_i(\C_2')$ have an exponential decay, i.e $ \lambda_j(\C'_2) = \alpha^j$ for some $\alpha < 1$, we have 
\begin{align*}
    \sum_{j>h} \lambda_j(\C'_2) & = \frac{\alpha^{h+1}}{1-\alpha}
\end{align*}
Set $h= \lceil-\operatorname{log}_\alpha(m)\rceil- 1$, we get
$$\sum_{j>h} \lambda_j(\C'_2) \leq \frac{1}{(1-\alpha)m}$$
Now,  
\begin{align*}
    \kappa(B_k, k,m) &= \inf_{h\geq 0} \left\{ \frac{B_k h}{m} + \sqrt{\frac{k \sum_{j>h}  \lambda_j(C_2')}{m}}\right\} \\ 
    &\leq \frac{-B_k \operatorname{log}_\alpha m}{m} + \frac{k}{(1-\alpha)m} \\
    & =  \frac{B_k \log{m}}{\log{1/\alpha}m} + \frac{k}{(1-\alpha)m} 
\end{align*}
where the last equality follows from the identity $\operatorname{log}_ba={\dfrac {\operatorname{log}_{d}(a)}{\operatorname{log} _{d}(b)}}$ \\
So essentially,  $\kappa (B_k, k,m)  = O \left(\frac{\log m}{m}\right)$. Therefore, we get 
\begin{align*}
    \epsilon_a = \ip{\fI P^k_\C}{\fI \C}_{HS(\rho)}- \ip{\fA P^k_{\C_m}}{\fI \C}_{HS(\rho)} &\leq \frac{24 B_k \log{m}}{\log{1/\alpha}m} + \frac{k}{(1-\alpha)m} + \frac{11\log \delta \tau^2 + 7B_k}{m} \\
    &=\frac{24 B_k \log{m}}{\log{1/\alpha}m} + \frac{k + (1-\alpha)(11\log \delta \tau^2 + 7B_k)}{(1-\alpha)m}
\end{align*}
which completes the proof.
\end{proof}

\subsection{Estimation Error}

We first remind the reader that $\fA P^k_{C_m}$ is a projection operator in $L^2(\cX,\rho)$ (See Lemma \ref{lem:LisProjA}).  However, the problem we face is that $\fA \hat{P}$ might not be a projection operator in $L^2(\cX,\rho)$. This is because the lifting is accomplished by lifting a particular set of eigenvectors of $\hat{\P}_\cA$ through $\sA$, and we remark that $\sA$ doesn't necessarily preserve norms and angles between elements. To get around this predicament, we show that lifted operator converges to a projection operator, i.e the lifted set of eigenvectors go to an orthogonal set of functions in $L^2(\cX,\rho)$. Moreover, from Lemma \ref{lem:eq6}, we have that this convergence in $HS(\rho)$ is equivalent to convergence in $HS(\cH)$. 
\begin{lemma} \label{lemma:est1}
When the number of samples  $n \geq \frac{2 \lambda_1^2 \fq{\cA}^2}{\lambda_k^2 (\sqrt{2}-1)}$, with probability at least $1-\frac{\delta}{2}$, we have
\begin{align*}
d(\fA \C_m, \cP^k_{HS(\rho)}) \leq \norm{\fA P^{k}_{\C_m} - \fA \hat{P}_\cA}_{HS(\rho)} &\leq \norm{\fA P^{k}_{\C_m} - \fA \hat{P}_\cA}_{\cL^1(\rho)} \\& \leq \frac{\lambda_1}{(\sqrt{2}-1)} \sqrt{\sum_{i=1}^k  \left(\frac{2 \lambda_i + 4 \lambda_1}{\lambda_i^2}\right)^2} \frac{\fq{\cA}}{n}
\end{align*}
\end{lemma}
\begin{proof}[Proof of Lemma \ref{lemma:est1}]

Since $\fA \C_m$ is a rank $k$ projection operator in $HS(\rho)$ (from Lemma \ref{lem:LisProjA}), the first inequality follows trivially. The second inequality is just from the property of norms that schatten norms $\norm{\D}_{\cL^p(\rho)}$  decreases with increasing $p$.  We focus on proving the third inequality below.

Let $\hat{\P}_\cA =  \tilde{\Phi}\tilde{\Phi}^\top$ be an eigendecomposition of the output $\hat{\P}_\cA$.  Let $$\RR^* = \argmin_{\RR^\top\RR = \RR \RR^\top = \I} \norm{\tilde{\Phi} \RR - \Phi_k}^2_F$$ 
 where $\Phi_k$ is the matrix corresponding top top $k$ eigenvectors of $\C_m$. 
Define $\hat{\Phi} := \tilde{\Phi} \RR^*$. This means that we rotate the eigenvectors of our output to a basis such that it is closest to the truth (in element-wise metric sense). An important point on why we can do this is that this rotation (or any other rotation for that matter) doesn't change the output, i.e. $\hat{\Phi}\hat{\Phi}^\top = \tilde{\Phi} \RR^*{\RR^*}^\top \tilde{\Phi}^\top  =\tilde{\Phi} \tilde{\Phi}^\top = \hat{\P}_\cA$. We now lifting the output by lifting this rotated set of eigenvectors. We have,  $\fA \hat{\P} = \sum_{i=1}^k \frac{\sA \hat{\phi}_i}{\sqrt{\hat{\lambda}_i}} \otimes_\rho \frac{\sA \hat{\phi}_i}{\sqrt{\hat{\lambda}_i}}$, where 
$\hat\lambda_i := \ip{\hat\phi_i}{\C_m \hat\phi_i}_\cF$.

\begin{align*}
\norm{\fA P^{k}_{C_m} - \fA \hat{P}_\cA}_{\cL^1(\rho)} &= \norm{\sum_{i=1}^k\frac{\sA \phi_i}{\sqrt{\lambda_i}} \otimes \frac{\sA \phi_i}{\sqrt{\lambda_i}}- \sum_{i=1}^k\frac{\sA \hat{\phi_i}}{\sqrt{\hat{\lambda_i}}} \otimes \frac{\sA \hat{\phi_i}}{\sqrt{\hat{\lambda_i}}}}_{\cL^1(\rho)} \\
& = \norm{\sA \sum_{i=1}^k  \left( \frac{\phi_i}{\sqrt{\lambda_i}} \otimes  \frac{\phi_i}{\sqrt{\lambda_i}} - \frac{\hat{\phi_i}}{\sqrt{\hat{\lambda_i}}} \otimes \frac{\hat{\phi_i}}{\sqrt{\hat{\lambda_i}}} \right) \sA^*}_{\cL^1(\rho)} \\
& \leq \norm{\sA}\norm{\sum_{i=1}^k  \left( \frac{1}{\lambda_i}{\phi_i} \otimes  {\phi_i} - \frac{1}{\hat{\lambda_i}}{\hat{\phi_i}} \otimes {\hat{\phi_i}} \right) \sA^*}_{\cL^1(\cF,\rho)} \\
& \leq \norm{\sA }  \norm{\sA^* }\norm{\sum_{i=1}^k  \left( \frac{1}{\lambda_i}{\phi_i} \otimes  {\phi_i} - \frac{1}{\hat{\lambda_i}}{\hat{\phi_i}} \otimes {\hat{\phi_i}} \right)}_{\cL^1(\cF)} \\
& \leq \lambda_1 \norm{\sum_{i=1}^k  \left( \frac{1}{\lambda_i}{\phi_i} \otimes  {\phi_i} - \frac{1}{\hat{\lambda_i}}{\hat{\phi_i}} \otimes {\hat{\phi_i}} \right)}_{\cL^1(\cF)}
\end{align*}

Where third and fourth inequalities follows from the fact that for trace-class operators $\norm{AB}_{\cL^1} \leq \norm{A}_2\norm{B}_{\cL^1}$. See  \citep[Exercise 28, Page 218]{reedmethods}.

Adding and subtracting $\frac{1}{\lambda_i}\hat\phi_i \otimes \hat\phi_i$ inside the summation to get

\begin{align*}
    & \leq \lambda_1 \norm{\sum_{i=1}^k  \frac{1}{\lambda_i}{\phi_i} \otimes  {\phi_i} - \frac{1}{\lambda_i}\hat{\phi_i} \otimes  \hat{\phi_i} + \frac{1}{\lambda_i}\hat{\phi_i} \otimes  \hat{\phi_i} - \frac{1}{\hat{\lambda_i}}{\hat{\phi_i}} \otimes {\hat{\phi_i}}}_{\cL^1(\cF)} \\
    & \leq \lambda_1 \norm{\sum_{i=1}^k \left( \frac{1}{\lambda_i} \left({\phi_i} \otimes  {\phi_i} - \hat{\phi_i} \otimes  \hat{\phi_i}\right) + \left(\frac{1}{\lambda_i} - \frac{1}{\hat{\lambda_i}} \right){\hat{\phi_i}} \otimes {\hat{\phi_i}} \right)}_{\cL^1(\cF)} \\
    & \leq \lambda_1 \sum_{i=1}^k \frac{1}{\lambda_i}\norm{{\phi_i} \otimes  {\phi_i} - \hat{\phi_i} \otimes  \hat{\phi_i}}_{\cL^1(\cF)} + \left| \frac{1}{\lambda_i} - \frac{1}{\hat{\lambda_i}}\right| \\
    & \leq \lambda_1 \sum_{i=1}^k \frac{1}{\lambda_i}\norm{{\phi_i} \otimes  {\phi_i} - \hat{\phi_i} \otimes  \hat{\phi_i}}_{\cL^1(\cF)} + \left| \frac{\lambda_i - \hat{\lambda_i}}{\lambda_i \hat{\lambda_i}}\right| \\
     & \leq \lambda_1 \sum_{i=1}^k  \frac{2}{\lambda_i}\norm{\phi_i - \hat{\phi_i}}_2 + \frac{4 \lambda_1}{\lambda_i^2}\norm{\phi_i - \hat{\phi_i}}_2 \\
    & \leq \lambda_1 \sum_{i=1}^k  \left(\frac{2 \lambda_i + 4 \lambda_1}{\lambda_i^2}\right)\norm{\phi_i - \hat{\phi_i}}_2 \\
    &  \leq \lambda_1 \sqrt{\sum_{i=1}^k  \left(\frac{2 \lambda_i + 4 \lambda_1}{\lambda_i^2}\right)^2} \norm{\Phi_k - \hat{\Phi}}_\cF  \\
     & \leq  \frac{\lambda_1}{2(\sqrt{2}-1)} \sqrt{\sum_{i=1}^k  \left(\frac{2 \lambda_i + 4 \lambda_1}{\lambda_i^2}\right)^2}\norm{\P^k_{\C_m}-\hat{\P} }^2_F  \\
    & \leq \frac{\lambda_1}{(\sqrt{2}-1)} \sqrt{\sum_{i=1}^k  \left(\frac{2 \lambda_i + 4 \lambda_1}{\lambda_i^2}\right)^2} \frac{\fq{\cA}}{n}
\end{align*}
The second to last inequality follows from Lemma \ref{lemma:est2} and Lemma \ref{lemma:est3}.
\end{proof}
\begin{lemma} \label{lemma:est2}
    $\|{\phi_i} \otimes  {\phi_i} - \hat{\phi_i} \otimes  \hat{\phi_i}\|_{\cL^1(\cF)} \leq 2 \|\phi_i-\hat\phi_i\|_2 \ \forall \ i \in [k]$
\end{lemma}
\begin{proof}
\begin{align*}
\|{\phi_i} \otimes  {\phi_i} - \hat{\phi_i} \otimes  \hat{\phi_i}\|_{\cL^1(\cF)} &= \|{\phi_i} \otimes  {\phi_i} - {\hat\phi_i} \otimes  {\phi_i} + \hat{\phi_i} \otimes  \phi_i -\hat{\phi_i} \otimes  \hat{\phi_i}\|_{\cL^1(\cF)} \\
 &\leq \|({\phi_i}-\hat\phi_i) \otimes {\phi_i}\|_{\cL^1(\cF)} + \|\hat{\phi_i} \otimes ({\phi_i}- \hat{\phi_i})\|_{\cL^1(\cF)} \\
&= 2 \|\phi_i-\hat\phi_i\|_2
\end{align*}
\end{proof}

\begin{lemma} \label{lemma:est3} When the number of samples $n \geq \frac{2 \fq{\cA}^2 \lambda_1^2}{\lambda_i^2 (\sqrt{2}-1)}$, with probability at least $1 - \delta, \ \forall \ i \in [k]$ we have,
    $$\left| \frac{\lambda_i-\hat\lambda_i}{\lambda_i\hat{\lambda_i}} \right| \leq  \frac{4 \lambda_1}{\lambda_i^2} \norm{\phi_i - \hat{\phi_i}}_2$$
where $C_\cA$ is a constant specific to the algorithm $\cA$.
\end{lemma}
The numerator is bounded as follows
\begin{proof}
\begin{align*}
\vert \lambda_i-\hat\lambda_i\vert &= \vert\phi_i^\top\C_m\phi_i - \hat\phi_i^\top\C_m \hat\phi_i \vert\\
&= \vert\phi_i^\top\C_m \phi_i - \phi_i^\top\C_m\hat\phi_i + \phi_i^\top\C_m\hat\phi_i - \hat\phi_i^\top\C_m\hat\phi_i\vert \\  
&= \vert\phi_i^\top\C_m (\phi_i-\hat\phi_i) + (\phi_i-\hat\phi_i)^\top\C_m\hat\phi_i\vert \\
&\leq  \|\C_m \phi_i\|_2 \|\phi_i-\hat\phi_i\|_2 + \|\phi_i-\hat\phi_i\|_2 \|\C_m\hat\phi_i\|_2 \\
&= (\lambda_i + \hat\lambda_i) \|\phi_i-\hat\phi_i\|_2 \\
& \leq  ( \lambda_i + \lambda_1) \|\phi_i-\hat\phi_i\|_2 \\
& \leq  2 \lambda_1 \|\phi_i-\hat\phi_i\| \\
\end{align*}
where the second inequality holds since $\hat\lambda_i < \lambda_1$ by definition of $\hat\lambda_i$, and the last inequality follows because $\hat{\lambda}_i \leq \lambda_1$. \\
The denominator is lower bounded similarly as
\begin{align*}
    \lambda_i\hat{\lambda_i} &\geq \lambda_i(\lambda_i - 2 \lambda_1 \|\phi_i-\hat\phi_i\|) \\
    & \geq \frac{\lambda_i^2}{2}
\end{align*}
where the first inequality follows from the bound on the numerator and the last inequality follows when $2 \lambda_1\norm{\phi_i - \hat{\phi}_i} \leq \frac{\lambda_i}{2}$. From Lemma \ref{lem:est6}, we know that 
$\norm{\phi_i - \hat{\phi_i}}_2^2 \leq \frac{1}{2(\sqrt{2}-1)} \left( \dfrac{\fq{\cA}}{n}\right)$ with probability at least $1 -\delta$. Combining, we get, with probability at least $1 -\delta$,
\begin{align*}
    \norm{\phi_i - \hat{\phi_i}}_2 \leq \sqrt{\frac{1}{2(\sqrt{2}-1)} \left( \frac{\fq{\cA}}{n}\right)} \leq \frac{\lambda_i}{4\lambda_1}
\end{align*}
The above holds when the number of samples $n \geq \frac{2 \lambda_1^2 \fq{\cA}^2}{\lambda_i^2 (\sqrt{2}-1)}$. 
Combining, we get
\begin{align*}
    \left\vert \frac{\lambda_i-\hat\lambda_i}{\lambda_i\hat{\lambda_i}}\right\vert \leq  \frac{4 \lambda_1}{\lambda_i^2} \norm{\phi_i - \hat{\phi_i}}_2
\end{align*}
with probability at least $1 - \delta$ and when $n \geq \dfrac{2 \lambda_1^2 \fq{\cA}^2}{\lambda_i^2 (\sqrt{2}-1)}$.
\end{proof}

Note that in particular since Oja's algorithm has a warm-up phase, the lower bound on the denominator

\begin{lemma} \label{lem:est4}
For rank $k$ orthogonal matrices $\U \in \R^{m \times k}$ and $\V \in \R^{m \times k}$, i.e.  $\U^\top \U = \V^\top \V = \I_k$, the following holds,
$$\norm{\U - \V \hat{\RR}}_F^2 \leq \frac{1}{2(\sqrt{2}-1)} \norm{\U\U^\top - \V\V^\top}_F^2,$$ 
where $$\hat{\RR} = \argmin_{\RR^\top\RR = \RR \RR^\top = \I_k} \norm{\U -\V \RR}^2_F$$
\end{lemma}

\begin{proof}
Proof in \citep[Lemma 6]{ge2017no}.
\end{proof}

Since $\Phi_k$ and $\hat{\Phi}$ are rank $k$ orthogonal matrices, from Lemma \ref{lem:est4}, we have $$\norm{\Phi_k - \hat{\Phi} }^2_F \leq \frac{1}{2(\sqrt{2}-1)}\norm{\P^k_{\C_m}-\hat{\P} }^2_F$$

\begin{lemma}
\label{lem:est5}
   For any efficient subspace learner $\cA$, we have  $\norm{\P^k_{\C_m}-\hat{\P} }^2_F \leq  \frac{2 \fq{\cA}}{n} $ with probability at least $1-\frac{\delta}{2}$.
\end{lemma}
\begin{proof}
\begin{align*}
\norm{\P^k_{\C_m}-\hat{\P}_\cA}^2_F &= \norm{\P^k_{\C_m}}_F^2 +\norm{\hat{\P}}_F^2 - 2\ip{\hat{\P}}{\P^k_{\C_m}} \\
& = 2\left(k - \ip{\hat{\P}}{\P^k_{\C_m}}\right)\\
& = 2\left(\ip{\I - \P^k_{\C_m}}{\hat{\P}}\right)\\
& = 2 \norm{\left(\Phi_k^\perp\right)^\top \hat{\Phi}}_F^2 \\
& \leq  \frac{2 \fq{\cA}}{n} 
\end{align*}
where the last inequality follows from the definition of efficient subspace learner.
\end{proof}

\begin{lemma} \label{lem:est6}
 With probability at least $1-\delta$,
    $$\norm{\phi_i - \hat{\phi_i}}_2 \leq \frac{1}{2(\sqrt{2}-1)} \left( \frac{\fq{\cA}}{n}\right) $$
    where $\fq{\cA}$ is specific to the algorithm $\cA$.
\end{lemma}
\begin{proof}
\begin{align*}
    \norm{\phi_i - \hat{\phi_i}}_2^2 & \leq  \sum_{i=1}^k\norm{\phi_i - \hat{\phi_i}}_2^2 \\
    & = \norm{\Phi_k - \hat{\Phi}}_F^2 \\
    & \leq \frac{1}{2(\sqrt{2}-1)}\norm{\P^k_{\C_m}-\hat{\P}_\cA}^2_F \\
    & \leq \frac{1}{2(\sqrt{2}-1)} \left( \frac{\fq{\cA}}{n}\right) \\
\end{align*}
\end{proof}
where the second inequality holds from Lemma \ref{lem:est4} and the definition of $\hat{\P}$, and the last inequality holds from Lemma \ref{lem:est5}

\begin{lemma}[Estimation Error] 
\label{lemma:estimationErrorAppendix}
When the number of samples  $n \geq \dfrac{2 \lambda_1^2 \fq{\cA}^2}{\lambda_k^2 (\sqrt{2}-1)}$, then with probability at least $1 - \delta$, we have
    $$\epsilon_{e} \leq \frac{\lambda_1^2}{(\sqrt{2}-1)} \sqrt{\sum_{i=1}^k  \left(\frac{2 \lambda_i + 4 \lambda_1}{\lambda_i^2}\right)^2} \frac{\fq{\cA}^2}{n}$$
\end{lemma}
\begin{proof}
\begin{align*}
    \epsilon_{e} &= \langle \fA P^k_{\C_m}, \fI \C \rangle_{HS(\rho)} -  \langle \fA \hat{P}_\cA, \fI \C \rangle_{HS(\rho)}  \\
    & = \ip{\fA P^k_{\C_m} - \fA \hat{P}_\cA}{\fI\C }_{HS(\rho)} \\
    & \leq \norm{\fA P^k_{\C_m} - \fA \hat{P}_\cA}_{\mathcal{L}^1(\rho)}\norm{\fI\C}_2 \\
    & \leq \lambda_1\norm{\fA P^k_{\C_m} - \fA \hat{P}_\cA}_{\mathcal{L}^1(\rho)} \\
    & \leq \frac{\lambda_1^2}{(\sqrt{2}-1)} \sqrt{\sum_{i=1}^k  \left(\frac{2 \lambda_i + 4 \lambda_1}{\lambda_i^2}\right)^2} \frac{\fq{\cA}}{n}
\end{align*}
The last inequality follows from Lemma \ref{lemma:est1}.
\end{proof}

We now invoke the approximation and the estimation error bounds i.e. Lemma \ref{lemma: approxError_Appendix} and Lemma \ref{lemma:estimationErrorAppendix} with failure probabilities $\delta/2$ each. We then apply a union bound over them and get that with probability at least $1-\delta$,
\begin{align*}
    \langle\fI\P_\C^k, \fI\C \rangle_\rho - \langle\fA\hat\P_\cA, \fI\C\rangle_\rho \leq \frac{c B_k}{\sqrt{n}} + \frac{c'(k +\log {\delta/2} + 7B_k)}{\sqrt{n} \log n} + \sqrt{\frac{\fqq{\cA}{2}}{n}},
\end{align*}
This concludes the proof of the main theorem. \\

Also note that since $d(\fA \hat\P_\cA, \cP_{HS(\rho)})$ decays as $O(1/\sqrt{n})$, we can bound the sub-optimality of $\fA \hat\P_\cA$ projected onto the set of projection operators $\cP^k_{HS(\rho)}$. 
It is now easy to give a bound on the objective with respect to the projection $\tilde\P_{\cA} \in \cP_{HS(\rho)}$  of $\fA\hat\P_{\cA}$ onto the set of projection operators:
\begin{corollary}
\label{cor:main_th_cor_appendix}
Let $\tilde\P_{\cA}$ be the projection of $\fA\hat\P_{\cA}$ onto the set $\cP_{HS(\rho)}$. Under the same conditions as in theorem~\ref{thm:main_th}, we have
\begin{align*}
\langle\fI\P_\C^k, \fI\C \rangle_\rho - \langle\tilde\P_{\cA}, \fI\C\rangle_\rho \leq 24\kappa(B_k,k,m) + \frac{\log {\delta/2} + 7B_k}{m} + 2\sqrt{\frac{\fqq{\cA}{2}}{n}}
\end{align*}
\end{corollary}
\begin{proof}
\begin{align*}
\langle\fI\P_\C^k, \fI\C \rangle_\rho - \langle\tilde\P_{\cA}, \fI\C\rangle_\rho &= \langle\fI\P_\C^k, \fI\C \rangle_\rho - \langle\fA\hat\P_{\cA}, \fI\C\rangle_\rho + \langle\fA\hat\P_{\cA}-\tilde\P_{\cA},\fI\C\rangle\\
&\leq 24\kappa(B_k,k,m) + \frac{\log {\delta/2} + 7B_k}{m} + \sqrt{\frac{\fqq{\cA}{2}}{n}}\\
&+ d\left(\fA \hat\P_\cA, \cP_{HS(\rho)}\right)\|\fI\C\|_{HS(\rho)}\\
&\leq 24\kappa(B_k,k,m) + \frac{\log {\delta/2} + 7B_k}{m} + 2\sqrt{\frac{\fqq{\cA}{2}}{n}},
\end{align*}
where the second to last inequality follows from Cauchy-Schwartz in $HS(\rho)$.
\end{proof}
We now give the proof of Corollary \ref{cor:main_good_decay}.
\begin{proof}[Proof of Corollary \ref{cor:main_good_decay}]
\begin{align*}
    \langle\fI\P_\C^k, \fI\C \rangle_\rho - \langle\fA\hat\P_\cA, \fI\C\rangle_\rho  &= \langle \fI \P^k_\C , \fI \C \rangle_{HS(\rho)} - \langle \fA \P^k_{\C_m}, \fI \C \rangle_{HS(\rho)} \\ &+ \langle \fA \P^k_{\C_m}, \fI \C \rangle_{HS(\rho)} -  \langle \fA \hat{\P}, \fI \C \rangle_{HS(\rho)} \\
    &\leq \frac{24 B_k \log{m}}{\log{1/\alpha}m} + \frac{k + (1-\alpha)(11\log{\delta/2} M + 7B_k)}{(1-\alpha)m} \\ &+ \lambda_1 \sqrt{\frac{\fqq{\cA}{2}}{n}},
\end{align*}
with probability at least $1-\delta$.  The last inequality follows from Lemma \ref{lemma: approxErrorGoodDecay} and Lemma \ref{lemma:estimationErrorAppendix} with a union bound over them.
\end{proof}

\newpage
\section{Examples of ESL}
In this section, we instantiate our framework with two popular learning algorithms, Empirical Risk Minimization (ERM) and Oja's Algorithm, and show that they satisfy the requirements of ESL.
\subsection{Empirical Risk Minimizer}
A natural candidate for an efficient subspace learner is the Empirical Risk Minimizer, which we call as $\cA_{ERM}$. We first show that $\cA_{ERM}$ satisfies the sufficient condition of Definition \ref{def:efficient-subspace-learner-main} and then show that $\cA_{ERM}$ is an efficient subspace learner. We then discuss its computational aspects. Let $\{\x_i\}_{i=1}^n$ be $n$ data samples and $\{\z(\x)\}_{i=1}^n$ be the corresponding representations in $\cF$. The empirical covariance matrix in $\cF$ is defined as
$$\hat{\C}_m = \frac{1}{n}\sum_{i=1}^n\z(\x_i) \z(\x_i)^\top$$

The algorithm $\cA_{ERM}$ computes the top $k$ eigenvectors of $\hat{\C}_m$, and returns a rank $k$ orthogonal matrix say $\hat{\Phi}$. Let the corresponding projection matrix be $\hat{\P}_{ERM}$. We first state the bound on covariance matrices $\C_m$ and $\hat{\C}_m$.

\begin{lemma}[Covariance Estimation] With probability at least $1-\delta$,
    $$\norm{\hat{\C}_m - \C_m}_2  \leq \frac{\kappa}{3n}\log{\frac{\delta}{2m}} + \sqrt{\frac{\kappa}{3n} \log{\frac{\delta}{2m}}^2 + \log{\frac{\delta}{2m}}\frac{\kappa\lambda_1}{n}}$$
\end{lemma}
\begin{proof}
\begin{align*}
    \norm{\hat{\C}_m - \C_m}_2 & = \norm{\frac{1}{n}\sum_{i=1}^n \z(\x_i) \z(\x_i) ^\top - \C_m}_2 \\
    & = \norm{\frac{1}{n}\sum_{i=1}^n\left(\z(\x_i) \z(\x_i) ^\top - \C_m \right)}_2 \\
    & = \norm{\sum_{i=1}^n \Xi_i}_2
\end{align*}
where $\Xi_i = \frac{1}{n}\left(\z(\x_i) \z(\x_i) ^\top - \C_m\right)$. $\Xi_i$'s are $0$ mean random matrices, i.e.  $\expect{\Xi_i} = 0 \ \forall \ i \in [n]$. \\
Note that $\norm{\z(\x)}^2_2 = \int_\cX \z_\omega(\x)^2 \d \rho(\x) \leq \tau^2$, since $\z_{\omega}(\x) \leq \tau \ \forall \ \omega \in \Omega, \x \in \cX$ by Assumption \ref{assumption:kernel}. We have,
\begin{align*}
    \norm{\Xi_i}_2 &\leq \frac{1}{n} \left(\norm{\z(\x_i) \z(\x_i)^\top}_2 + \norm{\C_m}_2\right) \\
    &= \frac{1}{n} \left(\trace{\z(\x_i) \z(\x_i)^\top} + \norm{\expectation{\x}{\z(\x)\z(\x)^\top}}_2\right) \\
    & \leq \frac{1}{n} \left(\norm{\z(\x_i)}^2 + \expectation{\x}{\norm{\z(\x)\z(\x)^\top}}_2\right) \\
      & \leq \frac{1}{n} \left(\norm{\z(\x_i)}^2 + \expectation{\x}{\norm{\z(\x)}^2_2}\right) \\
     & \leq \frac{2\tau^2}{n}
\end{align*} 
so that $L(\Xi):=\max_i\{\|\Xi_i \|_2 \}\leq \frac{2\tau^2}{n}$.
where in the second inequality, we apply Jensen's inequality.
Define $v(\Xi):=\norm{\sum_{i=1}^n \expect{\Xi_i \Xi_i^\top}}_2$. We have,
\begin{align*}
    v(\Xi) &= \norm{\sum_{i=1}^n \expect{\Xi_i \Xi_i^\top}}_2 \\
    & = \norm{\sum_{i=1}^n \expect{\frac{1}{n^2}\left(\z(\x_i) \z(\x_i) ^\top - \C_m\right)\left(\z(\x_i) \z(\x_i) ^\top - \C_m\right)^\top}}_2 \\
    & = \frac{1}{n^2}\norm{\sum_{i=1}^n \expect{   \norm{\z(\x_i)}^2 \z (\x_i)^\top \z (\x_i) - \z (\x_i)^\top \z (\x_i) \C_m - \C_m \z (\x_i)^\top \z (\x_i) + \C_m^2}}_2 \\
    & \leq \frac{1}{n^2}\norm{\sum_{i=1}^n \expect{   \tau^2 \z (\x_i)^\top \z (\x_i) - \z (\x_i)^\top \z (\x_i) \C_m - \C_m \z (\x_i)^\top \z (\x_i) + \C_m^2}}_2 \\
    & = \frac{1}{n^2}\norm{\sum_{i=1}^n \tau^2 \C_m -  \C_m^2 - \C_m^2 + \C_m^2}_2 \\
    & = \frac{1}{n}\norm{ \tau^2 \C_m -  \C_m^2}_2 \\
    & \leq  \frac{\tau^2 }{n}\norm{ \C_m}_2 \\
    & \leq \frac{\tau^2 \lambda_1}{n}
\end{align*}
where the second last inequality holds because $\C_m$ is a positive semi-definite matrix.
From matrix Bernstein concentration (Lemma \ref{thm:bernstein}, restated from \citep{tropp2015introduction}), we have, with probability at least $1-\delta$
\begin{align*}
    \norm{\hat{\C}_m - \C_m}_2 = \norm{\sum_{1}^m \Xi_i}_2 & \leq \frac{L(\Xi)}{6}\log{\frac{\delta}{2m}} + \sqrt{\frac{L(\Xi)^2}{12} \log{\frac{\delta}{2m}}^2 + \log{\frac{\delta}{2m}}v(\Xi)} \\
    & \leq \frac{\tau^2}{3n}\log{\frac{\delta}{2m}} + \sqrt{\frac{\tau^2}{3n} \log{\frac{\delta}{2m}}^2 + \log{\frac{\delta}{2m}}\frac{\tau^2\lambda_1}{n}} \\
\end{align*}
\end{proof}

In the following lemma, we show that $\cA_{ERM}$ is an efficient subspace learner.
\begin{lemma} \label{lemma:ERM}
$\cA_{ERM}$ is an effcient subspace learner.
\end{lemma}
\begin{proof}
We invoke Theorem \ref{thm:wedin} with the sub-multiplicative norm being the spectral norm. With $\A = \C_m, \B = \hat{\C}_m, \U = \hat{\Phi}, \V = \Phi_k^\perp$. Let $\epsilon = \norm{\C_m - \hat{\C}_m}_2$. .From Weyl's inequality, we have $\lambda_k(\hat{\C}_m) \geq  \lambda_k - \epsilon = \lambda_{k+1} + \gap  - \epsilon \geq \lambda_{k+1}$, if $\epsilon < \gap$. Therefore, setting $\mu = \lambda_{k+1}$, and $\alpha = \gap = \lambda_k - \lambda_{k+1}$, then with probability $1-\delta$, we get
\begin{align*}
 \norm{ (\Phi^\perp_k)^\top \hat{\Phi}}_F^2  &\leq k\norm{ (\Phi^\perp_k)^\top \hat{\Phi}}_2^2 \leq \frac{k\epsilon}{\alpha^2} \\
 & \leq  \frac{k}{\alpha^2}\left(\frac{\tau}{3n}\log{\frac{\delta}{2m}} + \sqrt{\frac{\tau}{3n} \log{\frac{\delta}{2m}}^2 + \log{\frac{\delta}{2m}}\frac{\tau\lambda_1}{n}}\right)^2 \\
 & \leq  \frac{k}{\alpha^2}\left(\frac{\tau}{3n}\log{\frac{\delta}{2m}} + \sqrt{\frac{2\lambda_1\tau}{n} \log{\frac{\delta}{2m}}}\right)^2 \\
  & \leq  \frac{k}{\alpha^2}\left(\frac{\lambda_1\tau^2}{n}\log{\frac{\delta}{2m}}^2\right) \\
\end{align*}
Setting $\fq{ERM} = \dfrac{\lambda_1\tau^2}{\alpha^2}\log{\dfrac{\delta}{2m}}^2 =\dfrac{k\lambda_1\tau^2}{(\lambda_k - \lambda_{k+1})^2}\log{\dfrac{\delta}{2m}}^2 $, we get,
 $$ \norm{ (\Phi^\perp_k)^\top \hat{\Phi}}_F^2 \leq \frac{\fq{ERM}}{n}$$
\end{proof}

\paragraph{Space and Computational Complexity of ERM:} ERM requires computing and storing the empirical covariance matrix $\hat{\C}_m$, which takes $O(m^2)$ memory. A rank $k$ SVD on $\hat{\C}_m$, generally, takes $O(m^2k)$ computations. We note that there are methods to scale this up but it is out of the scope of this work.
\subsection{Oja's Algorithm}
Having shown that ERM achieves optimal statistical rates, we now discuss a (relatively) more efficient algorithm in terms of space and computational complexity. We leverage the recent analysis of the classical Oja's algorithm and show how the algorithmic parameters affect the main result. We first restate the theorem statement from the analysis of Oja in  \cite{allen2016first}.

\begin{theorem} \label{theorem:oja}
Let $\gap := \lambda_k - \lambda_{k+1} \in \left(0,\frac{1}{k}\right]$ and $\Lambda := \sum_{i=1}^k \lambda_i \in (0,1]$, for every $\epsilon, \delta \in (0,1)$ define learning rates
$$T_0 = \Theta \left(\frac{4k \Lambda}{\gap^2 \delta^2}\right), T_1 = \Theta \left(\frac{ \Lambda}{\gap^2}\right), \eta_t = \begin{cases} \Theta \left(\frac{1}{\gap T_o}\right) & 1 \leq t \leq T_0 \\
\Theta \left(\frac{1}{\gap^2 T_1} \right) & T_0 < t \leq T_0 + T_1 \\
\Theta \left(\frac{1}{\gap (t-T_0)}\right) & t > T_0 + T_1\end{cases}$$

Let $Z$ be the column orthonormal matrix consisting of all eigenvectors of $\C_m$ with values no more than $\lambda_{k+1}.$ Then the output $\Q_T$ of the algorithm satisfies with at least $1-\frac{\delta}{2}$,
\[\text{for every } T = T_0 + T_1 + \Theta \left(\frac{T_1}{\epsilon}\right), \text{ it satisfies} \norm{\Z^\top\Q_T}_F^2 \leq \epsilon\]
\end{theorem}

The above theorem gives guarantees of the form required by the definition of efficient subspace learner. Therefore, implicitly, Oja is an efficient subspace learner. This is formally stated in the following lemma.

\begin{lemma}
$\cA_{oja}$ is an Efficient Subspace Learner.
\end{lemma}
\begin{proof}
From Theorem \ref{theorem:oja}, we have
\begin{align*}
    \norm{\Z^\top\Q_n}_F^2 &\leq \epsilon \\
    &=\tilde \Theta \left(\frac{T_1}{n-T_0-T_1} \right) \\
    & \leq \tilde \Theta \left( \frac{2\Lambda}{\gap^2 n}\right) \tag{for large $n$}
\end{align*}

Setting $\fq{oja} = \tilde \Theta \left(\frac{\Lambda}{\gap^2}\right)$, we get,
$$ \norm{\Z^\top\Q_n}_F^2 \leq \frac{\fq{oja}}{n}$$
\end{proof}
Moreover the requirement of an initial constant number of samples as stated in Theorem \ref{thm:main_th} also appears in Theorem \ref{theorem:oja} as warm-up phase. Therefore, the requirement of initial samples can be absorbed in the warm-up phase of Oja. 

\paragraph{Space and Computational Complexity of Oja's Algorithm: } Oja's algorithm takes $O(mk)$ memory. The per iteration computational cost is $O(mk)$. Therefore, for an $\epsilon$-suboptimal solution, the total computational cost is $O\left(\frac{mk}{\epsilon^2}\right)$.
\newpage
\section{Experiments}\label{sec:app:exp}
We now need some lemmas which gives us analytical forms which would be used  to calculate the objective with respect to empirical measure in the experiments.

Let $\hat{\P}_\cA$ be the output of an efficient subspace learner $\cA$. Let $\hat{\P}_\cA =\tilde{\Phi} \tilde{\Phi}^\top$ be its eigendecomposition. We define $\hat{\Phi} = \tilde{\Phi}\RR^*$, where let $$\RR^* = \argmin_{\RR^\top\RR = \RR \RR^\top = \I} \norm{\tilde{\Phi} \RR - \Phi_k}^2_\cF$$ 
The following gives gives an explit form for $\RR^*$.

\begin{lemma}\label{lem:exp1}
For any orthogonal matrix $\tilde{\Phi} \in \R^{m \times k}$ and $\Phi \in \R^{m \times k}$, the solution of the optimization problem $$\argmin_{\RR^\top\RR = \RR \RR^\top = \I} \norm{\tilde{\Phi} \RR - \Phi}^2_\cF$$ 
is $\RR^* =  \tilde{\Phi}^\top\Phi$
\end{lemma}
\begin{proof}
\begin{align*}
    \argmin_{\RR^\top\RR = \RR \RR^\top = \I} \norm{\tilde{\Phi} \RR - \Phi_k}^2_\cF &= \argmin_{\RR^\top\RR = \RR \RR^\top = \I} - \trace{\RR^\top \tilde{\Phi}^\top \Phi_k} \\
\end{align*}
\begin{align*}
       \max_{\RR^\top\RR = \RR \RR^\top = \I} \trace{\RR^\top \tilde{\Phi}^\top \Phi_k} & \leq \norm{\RR}_F\norm{\tilde{\Phi}^\top \Phi_k}_F = k
\end{align*}
Note that $\trace{\RR^* \tilde{\Phi}^\top\Phi_k} = k$. So, the maximum is achieved at $\RR = \RR^* = \tilde{\Phi}^\top \Phi_k$.
\end{proof}
We have $\hat{\Phi} = \tilde{\Phi}\RR^*$. We now use this and apply Lemma \ref{lem:exp2} to evaluate the objective. 
\begin{lemma}\label{lem:exp2}
For a projection matrix $ \P= \U \U^\top = \sum_{i=1}^k u_i \otimes_\cF u_i, \ip{\fA \P}{\fI C}_\rho = \frac{1}{n}\trace{\V^\top \K \V}$, where $\V = \Phi^\top \U \S^{-\frac{1}{2}}$ and $\S = \diag{\lambda_1,\lambda_2, \ldots, \lambda_k}, \lambda_i = \ip{\C_m u_i}{u_i}_{\cF}$
\end{lemma}

\begin{proof}[Proof of Lemma~\ref{lem:exp2}]
\begin{align*}
    \ip{\fA \P}{\fI \C}_{HS(\rho)} &= \ip{\sum_{i=1}^k \frac{\sA u_i}{\sqrt{\lambda_i}} \otimes_\rho \frac{\sA u_i}{\sqrt{\lambda_i}}}{\sum_{j =1}^n \bar{\phi_j} \otimes_\rho \bar{\phi_j}}_{HS(\rho)} \\
    & = \sum_{i,j=1}^{k,n} \frac{1}{\lambda_i}\ip{\sA u_i}{\bar{\phi_j}}_\rho^2 \\
     & = \sum_{i,j=1}^{k,n} \frac{1}{\lambda_i}\ip{u_i}{\sA^* \bar{\phi_j}}_\cF^2
\end{align*}
where the second equality follows from bi-linearity of inner products, third from the definition of adjoints.
\begin{align*}
    \ip{u_i}{\sA^* \bar{\phi}_j}_\cF & = \sum_{l=1}^m (u_i)_l \left( \sA^* \bar{\phi_j}\right)_l \\
    &=  \sum_{l=1}^m (u_i)_l \frac{1}{n} \sum_{q=1}^n \bar{\phi}_j (x_q) \z (x_q)_l   = \frac{1}{n} u_i^\top \Phi \bar{\Phi}_j
\end{align*}
where $\bar{\Phi_j} \in \R^n$ and $\left(\bar{\Phi_j}\right)_q = \bar{\phi_j} (x_q)$. \\

Note that 
\begin{align*}
    \bar{\lambda_j} &= \ip{\bar{\C \phi_j}}{\bar{\phi_j}}_\cH  = \ip{\sI^* \sI \bar{\phi_j}}{ \bar{\phi_j}}_\cH \\
    &=  \ip{\sI \bar{\phi_j}}{\sI \bar{\phi_j}}_\rho  =  \ip{\bar{\phi_j}}{ \bar{\phi_j}}_\rho \\
    &= \frac{1}{n}\sum_{q=1}^n  \bar{\phi_j}(x_q)^2  = \frac{1}{n} \norm{\bar{\Phi_j}}_2^2
\end{align*}
where the third equality follows from the property of adjoints. \\
Moreover, $ \left(\V^*_j\right)^\top \K \V^*_j = \bar{\lambda_j}$. Therefore $\bar{\Phi_j}^\top \bar{\Phi_j} = n \bar{\lambda_j} = n\left(\V^*_j\right)^\top \K \V^*_j$. \\
So we have $\bar{\Phi_j} = 
\sqrt{n}\K^{1/2} \V_j^*$. Hence,
\begin{align*}
    \ip{\fA \P}{\fI \C}_{HS(\rho)} &= \sum_{i,j=1}^{k,n} \frac{1}{\lambda_i} \left(\frac{1}{n} u_i^\top \Phi\sqrt{n}\K^{1/2} \V_j^* \right)^2 \\
    &=\frac{1}{n} \sum_{i,j=1}^{k,n} \frac{1}{\lambda_i} \left( u_i^\top \Phi \K^{1/2} \V_j^*\right)^2 \\
    &=\frac{1}{n} \sum_{i,j=1}^{k,n} \left(\frac{1}{\sqrt{\lambda_i}}  u_i^\top \Phi \K^{1/2} \V_j^*\right)^2  \\
        &=\frac{1}{n} \sum_{i,j=1}^{k,n} \left( \V_i^\top \K^{1/2} \V_j^*\right)^2 
\end{align*}

where $\V = \Phi^\top \U \S^{-\frac{1}{2}}$. Therefore, we have, 
\begin{align*}
    \ip{\fA \P}{\fI \C}_{HS(\rho)} &= \frac{1}{n} \sum_{i,j=1}^{k,n} \trace{ \V_i^\top \K^{1/2} \V_j^*}^2 \\
    &= \frac{1}{n} \sum_{i,j=1}^{k,n} \trace{\V_i^\top \K^{1/2} \V_j^* (\V_j^*)^\top\K^{1/2}  \V_i} \\
     &= \frac{1}{n} \sum_{i=1}^k \trace{\V_i^\top \K^{1/2}\sum_{j=1}^n \V_j^* (\V_j^*)^\top\K^{1/2}  \V_i} \\
      &= \frac{1}{n} \sum_{i=1}^k \trace{\V_i^\top \K^{1/2} \V^* (\V^*)^\top\K^{1/2}  \V_i} \\
     &= \frac{1}{n} \sum_{i=1}^{k} \trace{ \V_i^\top\K \V_i } \\
     &= \frac{1}{n} \trace{\V^T\K\V}
\end{align*}
\end{proof}

\newpage
\section{Auxillary Results}
Here we state some Auxillary results used in the proofs.

\begin{theorem}
[Generalized Gap free Wedin Theorem] \label{thm:wedin}
For $\epsilon > 0$, let $\A$ and $\B$ be two PSD matrices. For every $\mu >0, \alpha > 0$, let $\U$ be column orthonormal matrix consisting of eigenvectors of $\A$ with eigenvalue $\leq \mu$, let $\V$ be column orthonormal matrix consisting of eigenvectors of $\B$ with
eigenvalue $\geq \mu + \alpha$ , then we have
$$\norm{\U^\top \V} \leq \frac{\norm{A-B}}{\alpha}$$
where the norm $\norm{\cdot}$ is any sub-multiplicative norm.
\end{theorem}

\begin{proof}
The above theorem is stated in \citep[Lemma B.3]{allen2016lazysvd} in the sense of spectral norm. For the sake of completeness, we present the proof and show that it can easily be generalized to any sub-multiplicative norm.

Let  the SVD of $\A$ and $\B$ be $\A = \U\Sigma U^\top + \U' \Sigma' \U^\top, \B = \V \tilde \Sigma \V^\top + \V' \tilde \Sigma' \V'^\top$, where $\Sigma$ is a diagonal matrix which contains all eigenvalues of $\A$ which are $\leq \mu$. Similarly, $\tilde \Sigma$ contains all eigenvalues $\geq \mu+\alpha$. Let $\E:= \A-\B$.
\begin{align*}
    & \Sigma \U^\top = \U^\top \A = \U^\top (\B + \E) \\
\end{align*}
where the first equality follows because $\U$ is orthogonal to $\U'$. Multiply by $\V$ on the right on both sides, we get,
\begin{align*}
    & \Sigma \U^\top \V = \U^\top \B \V +  \U^\top \E \V = \U^\top \V \tilde \Sigma +   \U^\top \E \V\\
\end{align*}
where the second equality follows because $\V$ is orthogonal to $\V'$. Multiplying by $\tilde \Sigma^{-1}$ on the right on both sides, we get,
\begin{align*}
    \Sigma \U^\top \V \tilde \Sigma^{-1} = \U^\top \V + \U^\top \E \V \tilde \Sigma^{-1}
\end{align*}

Taking any sub-multiplicative norm on the left hand side, we obtain an upper bound on it as follows, 
\begin{align*}
\norm{ \Sigma \U^\top \V \tilde \Sigma^{-1}} &\leq \norm{\Sigma}_2 \norm{\tilde \Sigma^{-1}}_2 \norm{U^\top \V} \\
&\leq \frac{\mu}{\mu+\alpha}\norm{\U^\top\V}
\end{align*}
where the first inequality follows from the property of sub-multiplicative norms, and the second from the definition of $\Sigma$ and $\tilde \Sigma$. \\
Similarly, taking any sub-multiplicative norm on the right hand side, we get a lower bound on it as follows,
\begin{align*}
\norm{\U^\top\V + \U^\top \E\V\tilde\Sigma^{-1}} &\geq \norm{\U^\top\V} - \norm{\U^\top\E\V\tilde \Sigma^{-1}} \\
&\geq \norm{\U^\top\V} - \norm{U^\top}_2 \norm{\E}\norm{\V}_2\norm{\tilde \Sigma^{-1}}_2 \\
& \geq  \norm{\U^\top\V} - \frac{\norm{\E}}{\mu+\alpha}
\end{align*}
where the first inequality follows from (reverse) triangle inequality, the second from property of sub-multiplicative norms and third because $\U$ and $\V$ are orthonormal matrices and by definition of $\tilde \Sigma$. \\
Combining both the bounds, we get,
\begin{align*}
& \norm{\U^\top \V}\left(1 - \frac{\mu}{\mu+\alpha}\right) \leq \frac{\norm{\E}}{\mu+\alpha} \\
& \implies \norm{\U^\top \V} \leq \frac{\norm{\E}}{\alpha}
\end{align*}
\end{proof}

\begin{theorem}[Matrix Bernstein \citep{tropp2015introduction}]
\label{thm:bernstein}
Let $\S_1,\S_2, \ldots \S_n$ be $n$ i.i.d $d_1 \times d_2$ random matrices such that $\mathbb{E}{\S_i} = \0, \norm{\S_i} \leq L \ \forall i \in [n]$. Let $\Z = \sum_{i=1}^n \S_i$. Let $v(\Z)$ denote the matrix variance statistic of the sum defined as,
$$v(\Z) = \max \{\mathbb{E}{\Z\Z^\top}, \mathbb{E}{\Z^\top \Z}\}$$ Then, with probability at least $1-\delta$, we have,
$$\bP\left\{\norm{\Z} \geq t\right\} \leq (d_1 + d_2) \exp{\frac{- t^2/2}{v(\Z) + Lt/3}} \ \forall t \ \geq 0$$
\end{theorem}

\begin{theorem}[Local Rademacher Complexity \citep{bartlett2002localized}] \label{thm:localrademacher}
Let $\cX$ be a measurable space. Let $\cP$ be a probability distribution on $\cX$ and let  $\x_1,\x_2 \ldots \x_n$ be i.i.d. samples drawn from $\cP$. Let $\cP_n$ denote the empirical measure. Let $\cF$ be a class of functions on $\cX$ ranging from $[-1,1]$ and assume that there exists some constant $B$ such that for every $f \in \cF, \cP^2 f \leq B \cP f$. Let $\psi$ be a sub-root function and let $r^*$ be the fixed point of $\psi$. If $\psi$ satisfies 
\begin{align*}
\psi(r) \geq B \expectation{X, \sigma}{\mathcal{R}_n \{f \in star(\cF) | \cP f^2 \leq r\}}
\end{align*}
where $star(\cF) = \{\lambda f | f \in \cF, \lambda \in [0,1]\}$ is the \textit{star shaped hull} of $\cF$ and $\mathcal{R}_n{\cF} = \sup_\f \in \cF \frac{1}{n}\sum_{i=1}^n \sigma_if (\x_i)$ is the empirical Rademacher complexity of $\cF$ given data points $\{\x_i\}_{i=1}^n$;
then for every $K> 0$ and $\x > 0$, with probability at least $1 - e^{-\delta}$
\begin{align}
\label{eqn:rademacher1}
\forall \ f \in \cF, \cP f \leq \frac{K}{K-1}\cP_n f + \frac{6K}{B}r^* + \frac{\delta(11+5BK)}{n}
\end{align}

Also, with probability at least $1- e^{-\delta}$
\begin{align}
\label{eqn:rademacher2}
\forall f \in \cF, \cP_n f \leq \frac{K}{K+1}\cP f + \frac{6K}{B}r^* + \frac{\delta(11 + 5BK)}{n}
\end{align}
Furthermore, if $\hat{\psi}_n$ is a data-dependent sub-root function with fixed point $\hat{r}^*$ such that
$$\psi^*(r) > 2 (10 \lor B) \expectation{\sigma}{\mathcal{R}_n \{ f \in star(\cF) | \cP^n f^2 \leq 2r\}} + \frac{2 (10 \lor B + 11)\delta}{n}$$
then with probability at least $1-2e^\delta$, it holds that $\hat{r}^* \geq r^*$; as a consequence, equations \ref{eqn:rademacher1} and \ref{eqn:rademacher2} holds with $r^*$ replaced by $\hat{r}^*$

\end{theorem}

\end{document}